%% file: sample_paper.tex
\documentclass[twoside]{article}
\usepackage[accepted]{aistats2026}
\usepackage[T1]{fontenc} 

\usepackage{graphicx} % Required for inserting images
\usepackage[utf8]{inputenc}
\usepackage{amsmath}
\usepackage{amsthm}
\usepackage{array}
\usepackage{balance}
\usepackage{multirow}
\usepackage{multicol}
\usepackage{microtype}
\usepackage{psfrag}
\usepackage{amssymb}
\usepackage{xcolor}
\usepackage{relsize}
\usepackage{mathtools}
\usepackage{setspace}
\usepackage{nccmath}
\usepackage{multirow}
\usepackage{multicol}
\usepackage{float}
\usepackage{amsthm}
\usepackage{dsfont}
\usepackage{hyperref}
\usepackage[hyphenbreaks]{breakurl} % Works with hyperref
\usepackage{xurl}  
\usepackage[font=small,labelfont=bf,
   justification=justified,
   format=plain]{caption}
\usepackage{subcaption}
\usepackage{algpseudocode}
\usepackage{algorithm}
\usepackage{acronym}
\newtheorem{theorem}{Theorem}
\newtheorem{remark}{Remark}
\newtheorem{corollary}{Corollary}
\newcommand\V[1]  { \mathbf{#1} }
\newcommand\B[1]  { \boldsymbol{#1} }

\newcommand\set[1] {\mathcal{#1}}

\newcommand{\up}[1]{\mathrm{#1}}

\newcolumntype{P}[1]{>{\centering\arraybackslash}p{#1}}
\DisableLigatures[<,>]{encoding=T1,family=tt*}

\acrodef{RRM}[RRM]{robust risk minimization}
\acrodef{ERM}[ERM]{empirical risk minimization}
\acrodef{MRC}[MRC]{minimax risk classifier}
\acrodef{LR}[LR]{logistic regression}
\acrodef{SOA}[SOA]{state-of-the-art}
\acrodef{RFF}[RFF]{random Fourier features}
\acrodef{DNN}[DNN]{deep neural network}
\acrodef{SGD}[SGD]{stochastic gradient descent}
\acrodef{MAE}[MAE]{mean absolute error}
\acrodef{MLP}[MLP]{multi-layer perceptron}
\acrodef{ECE}[ECE]{expected calibration error}
\acrodef{MGCE}[MGCE]{minimax generalized cross-entropy}
\usepackage[round]{natbib}

\begin{document}
\runningauthor{Kartheek Bondugula, Santiago Mazuelas, Aritz Pérez, Anqi Liu}



\twocolumn[

\aistatstitle{Minimax Generalized Cross-Entropy}

\aistatsauthor{
Kartheek Bondugula$^{1}$ \ \ \ \ \ \
Santiago Mazuelas$^{1,2}$ \ \ \ \ \ \
Aritz P\'erez$^1$ \ \ \ \ \ \
Anqi Liu$^3$
}

\aistatsaddress{
$^1$ Basque Center for Applied Mathematics, Bilbao, Spain \\
$^2$ Ikerbasque, Basque Foundation for Science \\
$^3$ Johns Hopkins University, Baltimore, USA \\
\texttt{\{kbondugula, smazuelas, aperez\}@bcamath.org}, \ \texttt{aliu.cs@jhu.edu}
}
]

\begin{abstract}
  Loss functions play a central role in supervised classification. \mbox{Cross-entropy} (CE) is widely used, whereas the \ac{MAE} loss can offer robustness but is difficult to optimize. 
  Interpolating between the CE and \ac{MAE} losses, generalized \mbox{cross-entropy} (GCE) has recently been introduced to provide a trade-off between optimization difficulty and robustness.
 Existing formulations of GCE result in a non-convex optimization over classification margins that is prone to underfitting, leading to poor performances with complex datasets.  In this paper, we propose a minimax formulation of generalized cross-entropy (MGCE) that results in a convex optimization over classification margins. Moreover, we show that MGCEs can provide an upper bound on the classification error. The proposed bilevel convex optimization can be efficiently implemented using stochastic gradient computed via implicit differentiation. Using benchmark datasets, we show that MGCE achieves strong accuracy, faster convergence, and better calibration, especially in the presence of label noise.
\end{abstract}

% \vspace{-0.2cm}
\section{INTRODUCTION}
% \vspace{-0.2cm}

The learning process in supervised classification is strongly influenced by the choice of loss function. Cross-entropy (CE) has long been the standard choice for classification tasks in machine learning. On the other hand, the \acf{MAE} loss is known to offer increased robustness to noise \citep{ghosh2017robust}, but is difficult to optimize \citep{zhang2018generalized}. To address this trade-off between optimization difficulty and robustness, a more general class of loss functions has been recently proposed, termed as generalized \mbox{cross-entropy} (GCE), which interpolate between the \ac{MAE} and CE losses through a tunable parameter \citep{zhang2018generalized, sypherd2022tunable}. These losses have been shown to alleviate the optimization challenges of the MAE loss while achieving improved robustness over CE \citep{zhang2018generalized}.

The GCE losses have been defined by leveraging probabilistic losses known as $\alpha$-losses, which vary depending on the value of $\alpha$ \citep{ferrari2010maximum, liao2018tunable, sypherd2020alpha, sypherd2022tunable, sypricklal:22}. Specifically, GCEs are obtained by first transforming classification margins into probabilities using a link function, and then using an $\alpha$-loss over the resulting probabilities. In existing methods, the same link function (softmax) is used for all values of $\alpha$, and the final classification loss becomes \mbox{non-convex} over margins. As a result, existing methods based on GCE are susceptible to underfitting, particularly when trained on complex datasets \citep{zhou2021asymmetric, staats2025enhancing}.

Recently, several works have proposed a minimax approach for classification with the 0-1 loss that results in a convex optimization over the classification margins \citep{asif2015adversarial, fathony2016adversarial, fathony2018consistent, MazRomGrun:23}. Since the \ac{MAE} loss coincides with the 0-1 loss of randomized classification rules, such approaches provide a convex formulation for the \ac{MAE} loss. Convexity is a desirable property for a loss function \citep{zhang2004statistical, bartlett2006convexity}, and convex surrogates often provide favorable optimization landscapes that improve convergence rates in large-scale machine learning \citep{agrawal2015}. However, it is still unclear whether it is possible to find a convex optimization solution to GCEs.

In this paper, we propose a minimax approach for generalized cross-entropy (MGCE) that results in a bilevel convex optimization over classification margins. The presented methods achieve the benefits provided by general $\alpha$-losses and avoid the underfitting issues of existing GCE methods. Specifically, the contributions of the paper can be summarized as follows:
\vspace{-0.3cm}
\begin{itemize}
    \item We present MGCE methods that minimize the worst-case expected \mbox{$\alpha$-loss} over distributions in an uncertainty set. 
    % In addition, we establish the multiclass Bayes consistency of MGCE for general $\alpha$-losses.
    \item We show that the MGCE method can provide an upper bound on the classification error, and that the MGCE margin loss is classification-calibrated. Moreover, we characterize the relation between the worst-case distribution in the uncertainty set and the corresponding minimax classifier for general $\alpha$-losses.
    \item We develop an efficient learning algorithm for the proposed bilevel convex optimization. Specifically, we derive the corresponding stochastic gradients using implicit differentiation, and propose a bisection-based algorithm to compute them efficiently, enabling scalability to complex models such as \acp{DNN}.
    \item Using large benchmark datasets and \acp{DNN}, we demonstrate that MGCE methods achieve strong accuracy and faster convergence. Additionally, the MGCEs yield better calibrated models than GCEs, particularly under label noise.
\end{itemize}

\textbf{Notations:} For a set $\set{S}$, we denote its cardinality as $|\set{S}|$; bold lowercase and uppercase letters represent vectors and matrices, respectively; for a vector $\V{b}$ and an index $i$, $\up{b}_i$ denotes the component at index $i$; for a vector-valued function $f(x)$ and an index $i$, $f(x)_i$ denotes the $i$-th component of $f(x)$; $\V{I}$ denotes the identity matrix; $\mathds{1} \{\cdot\}$ denotes the indicator function; $\B{1}$ denotes a vector of ones; for a vector $\V{v}$, $|\V{v}|$ and $(\V{v})_+$, denote its component-wise absolute value and positive part, respectively; for a scalar $\up{v}$, $(\up{v})_+^\beta$ denotes the positive part raised to the power of $\beta$; $\otimes$ denotes the Kronecker product; $\preceq$ and $\succeq$ denote vector inequalities; $\mathbb{E}_{\up{p}}\{\,\cdot\,\}$ denotes the expectation of its argument with respect to distribution $\up{p}$.

\vspace{-0.1cm}
\section{RELATED WORK}
\vspace{-0.2cm}
This section introduces the main terminology, reviews existing approaches for GCE, and briefly outlines the minimax framework of \cite{MazSheYua:22, MazRomGrun:23}.

\vspace{-0.1cm}
\subsection{Preliminaries}
\vspace{-0.1cm}
Let $\set{X}$ and $\set{Y}$ be the set of features and labels in a supervised classification problem with $k$ classes. The goal is to find a classifier $\up{h}: \set{X} \rightarrow [0,1]^k$ that maps each feature vector $x \in \set{X}$ to labels' probabilities. Specifically, given $x \in \set{X}$, we denote by $\up{h}(x)$ the vector of class probabilities  given $x$, where the probability of label $y \in \set{Y}$ is given by the $y$-th component of the vector, $\up{h}(x)_y$. Commonly, the probability vector is obtained using a link function, such as softmax function, over the classifier margins. The classifier margins over different classes are represented by the vector 
\begin{equation}
\label{eq:margins}
f(x, \B{\mu}) = [\Phi(x,1)^\top\B{\mu}, \Phi(x,2)^\top\B{\mu},\ldots,\Phi(x,k)^\top\B{\mu}], 
\end{equation}
where $\Phi(x,y) \in \mathbb{R}^m$ is the feature mapping corresponding to the instance-label pair $(x,y)$, e.g., one-hot encoding of the penultimate layer in a \ac{DNN}, and $\B{\mu} \in \mathbb{R}^m$ are coefficients learned by the classifier. The margin corresponding to class $y$ is the $y$-th component of $f(x,\B{\mu})$, i.e., $f(x,\B{\mu})_y=\Phi(x,y)^\up{T}\B{\mu}$.

We denote by $\Delta(\set{X} \times \set{Y})$ the set of probability distributions on $\set{X} \times \set{Y}$. 
% For a given probabilistic loss $\ell$, . 
% $\ell(\up{h}, \up{p}) = \mathbb{E}_{\up{p}}\ell(\up{h},(x,y))$ denotes the expected loss of the probabilistic classifier $\up{h}$ with respect to distribution $\up{p} \in \Delta(\mathcal{X}\times\mathcal{Y})$. 
If $\up{p}^*$ denotes the underlying distribution of the instance-label pairs, the classification risk of a rule $\up{h}$ under the loss $\ell$ is given by
\begin{equation}
\mathcal{R}_{\ell}(\up{h}) = \mathbb{E}_{\up{p}^*}\ell(\up{h},(x,y)).
\end{equation}
The most commonly used probabilistic loss, especially in modern \acp{DNN}, is the CE that is defined as $\ell_{\text{CE}}(\up{h}, (x,y)) = -\log\up{h}(x)_y$. However, such loss function can lead to overfitting and miscalibrated models \citep{guo2017calibration}. Besides CE, the \mbox{\ac{MAE}} loss function defined as $\ell_{\text{\ac{MAE}}}(\up{h}, (x,y)) = 1 - \up{h}(x)_y$ is more robust to noise while may pose significant challenges for optimization \citep{zhang2018generalized}. The \ac{MAE} loss coincides with the 0-1 loss of randomized classification rules. Specifically, if $\up{h}(x)_y$ is the probability with which rule $\up{h}$ assigns label $y$ to instance $x$, the value $1-\up{h}(x)_y$ is the probability of incorrectly labeling the example $(x,y)$.

% \vspace{-0.3cm}
\subsection{Generalized cross-entropy}
\vspace{-0.1cm}
The GCE losses have recently been proposed as a family of loss functions that interpolate between the CE and \ac{MAE} losses, thereby offering a trade-off between robustness and optimization difficulty \citep{zhang2018generalized, sypherd2020alpha, sypricklal:22}. GCE losses are obtained using a family of probabilistic losses referred to as $\alpha$-loss \citep{sypherd2020alpha} and defined as
\begin{equation}
    \ell_{\alpha}(\up{h}, (x, y)) = \frac{\alpha}{\alpha-1} (1 - \up{h}(x)_y^\frac{\alpha - 1}{\alpha}).
\end{equation}
The minimum expected $\alpha$-loss (Bayes risk) can be related to generalized notions of entropy as described in \cite{MazSheYua:22}. In particular, the Bayes risk corresponding to an $\alpha$-loss coincides with the Arimoto conditional entropy, which is closely related to Tsallis entropy.

For notational convenience, we  parametrize $\alpha$-losses using \mbox{$\beta=\alpha/(\alpha-1)$} leading to the following formulation
\begin{equation}
    \label{eq:loss_beta}
    \ell_{\beta}(\up{h}, (x, y)) = \beta (1 - \up{h}(x)_y^\frac{1}{\beta}).
\end{equation}
For $\beta \geq 1$, such probabilistic losses interpolate between the \ac{MAE} loss and the CE loss: specifically, $\beta = 1$ recovers the \ac{MAE} loss, while $\beta \to \infty$ yields the CE loss. Figure~\ref{fig:beta_loss} shows the variation of the loss function with respect to $\beta$. For $\beta \in [1,\infty)$, the corresponding loss is noise-robust, since the sum of the losses over all classes is bounded as
\begin{equation}
    \beta(k - k^\frac{{\beta - 1}}{\beta})\leq \sum_{y=1}^k\ell_{\beta}(\up{h}, (x, y)) \leq \beta(k - 1),
\end{equation} 
as shown in \cite{zhang2018generalized}.

Existing GCE methods obtain classification losses from the $\alpha$-losses in \eqref{eq:loss_beta} by transforming the classifier's margins to probabilities using a link function that does not depend on the value of $\alpha$ \citep{zhang2018generalized, sypherd2020alpha}. Such methods use the softmax link so that the probability of class $y$ for a given margin $f(x, \B{\mu})$ is obtained as \mbox{$\up{h}(x)_y = e^{f(x, \B{\mu})_y} / \sum_{i=1}^k e^{f(x, \B{\mu})_i}$}. Therefore, existing GCE methods address the following minimization at training
\begin{equation}
\label{eq:erm}
\min_{\B{\mu}} \sum_{i=1}^n \beta\Bigg( 1 - \frac{e^{f(x_i, \B{\mu})_{y_i} / \beta}}{\Big(\sum_{j=1}^{k} e^{f(x_i, \B{\mu})_j}\Big)^{1/\beta}} \Bigg),
\end{equation}
where $f(x_i, \B{\mu})_j$ denotes the classification margin for instance $x_i$ and class $j$ as defined in \eqref{eq:margins}.

\begin{figure}
    \centering
    \psfrag{12345678912345}[l][l][0.9]{CE}
    \psfrag{12345678912346}[l][l][0.9]{$\beta = 2.0$}
    \psfrag{12345678912347}[l][l][0.9]{$\beta=1.4$}
    \psfrag{12345678912348}[l][l][0.9]{$\beta=1.1$}
    \psfrag{12345678912349}[l][l][0.9]{\ac{MAE}}
    \psfrag{hyx}[c][t][0.9]{Class probability $\up{h}(x)_y$}
    \psfrag{loss}[c][t][0.9]{Loss $\ell(\up{h}, (x,y))$}
    \psfrag{0.2}[c][c][0.8]{0.2}
    \psfrag{0.4}[c][c][0.8]{0.4}
    \psfrag{0.6}[c][c][0.8]{0.6}
    \psfrag{0.8}[c][c][0.8]{0.8}
    \psfrag{1}[c][c][0.8]{1}
    \psfrag{0}[r][c][0.8]{0}
    \psfrag{2}[r][r][0.8]{2}
    \psfrag{4}[r][r][0.8]{4}
    \psfrag{3}[r][r][0.8]{3}
    \psfrag{5}[r][r][0.8]{5}
    \psfrag{10}[r][r][0.8]{10}
    \includegraphics[width=0.47\textwidth]{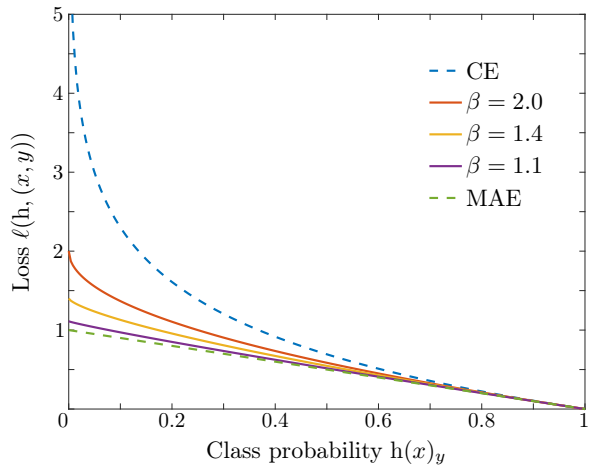}
    % \captionsetup{labelfont={it}, font=small}
    \caption{Relation between $\beta$ and the resulting loss function. For $\beta =1$, the loss corresponds to the \ac{MAE} while for $\beta=\infty$, it corresponds to CE. For $\beta \in(1,\infty)$, the loss interpolates between the \ac{MAE} and CE.}
    \label{fig:beta_loss}
    \vspace{-0.3cm}
\end{figure}

For $1 \leq \beta < \infty$, the classification loss in \eqref{eq:erm} is non-convex over classification margins \citep{zhang2018generalized, sypherd2022tunable}, resulting in optimization difficulties especially in complex datasets \citep{zhou2021asymmetric, staats2025enhancing}.

Convexity remains a highly desirable property for a loss function, as it can enable tractable optimization \citep{zhang2004statistical, bartlett2006convexity}. Minimax approaches have recently provided convex formulations for the 0-1 loss (and hence for the \ac{MAE} loss) \citep{asif2015adversarial, fathony2016adversarial, fathony2018consistent, MazZanPer:20, MazSheYua:22, MazRomGrun:23}. In the following, we briefly describe the minimax framework that will be used in Section~\ref{sec:MGCE} to derive the proposed MGCE methods for general $\alpha$-losses. 

\vspace{-0.1cm}
\subsection{Minimax framework}
\vspace{-0.2cm}
The minimax framework obtains classification rules by minimizing the worst-case expected loss over distributions in an uncertainty set \citep{MazSheYua:22, MazRomGrun:23}. Specifically, for a given $\ell$-loss function, the minimax classification rule $\up{h}_\ell$ is obtained by solving
\begin{equation}
\label{eq:minmaxrisk}
\up{V}_{\ell} = \underset{\up{h}}{\min} \, \underset{\up{p} \in \mathcal{U}}{\max} \; \mathbb{E}_{\up{p}}\ell(\up{h},(x,y)),
\end{equation}
where $\up{V}_{\ell}$ denotes the worst-case risk and  $\set{U}$ denotes an uncertainty set of distributions. The minimax classification rule  $\up{h}_\ell$ minimizes the expected loss with respect to a worst-case distribution in $\set{U}$. When the true underlying distribution $\up{p}^*$ is in the uncertainty set, the worst-case risk upper bounds the classification risk of the minimax rule $\up{h}_{\ell}$ under the $\ell$-loss, that is,
\begin{equation}
\label{eq:upper_bound_alpha_risk}
\mathcal{R}_{\ell}(\up{h}_{\ell}) \leq \up{V}_{\ell}.
\end{equation}

The uncertainty set $\set{U}$ in \eqref{eq:minmaxrisk} is defined as
\begin{align}
\label{eq:us}
\mathcal{U} = \{\up{p} \in \Delta (\mathcal{X} \times \mathcal{Y}) : |\mathbb{E}_{\up{p}}\{\Phi(x,y)\} - \B{\tau}| \preceq \B{\lambda} \nonumber \\ \text{and} \ \up{p}(x) = \up{p}^*(x)\},
\end{align}
with $\up{p}^*(x)$ denoting the underlying distribution over the instances. The vector $\B{\tau}$ denotes the expectation estimates associated with the feature mapping $\Phi$, and $\B{\lambda} \succeq \V{0}$ is a confidence vector that accounts for inaccuracies in these estimates. The mean and confidence vectors can be obtained from the training samples $\{(x_i,y_i)\}_{i=1}^n$ as
\begin{align}
\label{eq:tau}
\B{\tau}=\frac{1}{n}\sum_{i=1}^{n}\Phi(x_i,y_i), \  \B{\lambda}= \lambda_0\V{s},
\end{align}
where $\V{s}$ denotes the vector formed by the component-wise sample standard deviations of $\Phi$ in the training samples and $\lambda_0$ is a regularization parameter. In particular, taking \mbox{$\lambda_0 = \sqrt{(\log m + \log \frac{2}{\delta})/2}$}, the underlying distribution $\up{p}^*$ of the training samples lies in the uncertainty set with probability $1-\delta$ (see e.g., Theorem 4 in \cite{MazRomGrun:23}).

Previous works develop minimax methods as in \eqref{eq:minmaxrisk} for the \ac{MAE} and CE losses \citep{MazSheYua:22, MazRomGrun:23}. In addition, \cite{MazSheYua:22} explores the usage of $\alpha$-losses under a minimax formulation that leads to an optimization that cannot be efficiently addressed in large-scale scenarios (not amenable for stochastic gradient descent (SGD) methods). 

To address the limitations of the existing methods, we introduce a minimax formulation under the $\alpha$-loss that results in a convex optimization over the classification margins, and effectively scales with complex datasets.

%%%
\vspace{-0.1cm}
\section{MINIMAX GENERALIZED CROSS-ENTROPY}
\label{sec:MGCE}
\vspace{-0.2cm}
In this section, we present the MGCE formulation based on the minimax framework in \eqref{eq:minmaxrisk} and the $\alpha$-loss in \eqref{eq:loss_beta}. In addition, we provide theoretical results that establish performance guarantees as well as an interpretation in terms of the worst-case distribution in $\set{U}$. 

The MGCE formulation minimizes the optimization problem in \eqref{eq:minmaxrisk} over $\alpha$-losses that results in the following bilevel convex optimization problem (see \mbox{Appendix~\ref{sec:derive_mgce}} for the detailed proof)
\begin{equation}
\label{eq:mrc}
\begin{array}{c}
\up{V}_{\beta} = \underset{\B{\mu} \in \mathbb{R}^m} {\min} - \B{\tau}^\top\B{\mu} + \B{\lambda}^\top|\B{\mu}| - \mathbb{E}_{\up{p}^*(x)} \{\varphi_\beta(x, \B{\mu})\}, \\
 % \text{s.t.} \ \underset{y \in \set{Y}}{\sum}\Big(\frac{f(x_i, \B{\mu})_y + \nu_i}{\beta} + 1\Big)^{\beta}_{+} = 1,   i=1,2,\ldots,n.
\end{array}
\end{equation}
where the function $\varphi_\beta(x, \B{\mu})$ is obtained as 
\begin{equation}
    \label{eq:implicit_function}
    \begin{array}{cl}
         \varphi_\beta(x, \B{\mu}) =& \underset{\nu}{\max} \ \nu  \\
         & \text{s.t.} \  \underset{y \in \set{Y}}{\sum}\Big(\frac{f(x, \B{\mu})_y + \nu}{\beta} + 1\Big)^{\beta}_{+} \leq 1.
    \end{array}
\end{equation}

Since the function defining the constraint is monotonically increasing in $\nu$, the function $\varphi_\beta$ is implicitly defined over the parameters $\B{\mu}$ as
\setlength{\arraycolsep}{2pt} % default ~5pt
\renewcommand{\arraystretch}{1.5}
\begin{equation}
\label{eq:optimality_condition}
    F(x, \B{\mu}, \varphi_\beta(x,\B{\mu})) =1, 
\end{equation}
where $F(x, \B{\mu}, \varphi_\beta(x,\B{\mu})) = \underset{y \in \set{Y}}{\sum}\Big(\frac{f(x, \B{\mu})_y + \varphi_\beta(x,\B{\mu})}{\beta} + 1\Big)^{\beta}_{+}$.

Note that the optimization in \eqref{eq:mrc} is convex over $\B{\mu}$ for $\beta\in[1,\infty)$ since the function $\varphi_\beta(x,\B{\mu})$ is concave in $\B{\mu}$. Specifically, if $\varphi_\beta(x,\B{\mu}_1)=\nu_1$ and $\varphi_\beta(x,\B{\mu}_2)=\nu_2$ we have that for any $t \in [0,1]$, $t\nu_1+(1-t)\nu_2$ is feasible for \eqref{eq:implicit_function} taking $\B{\mu}=t\B{\mu}_1+(1-t)\B{\mu}_2$, because the constraint in \eqref{eq:implicit_function} is convex over $\B{\mu}$ ($f(x,\B{\mu})$ in \eqref{eq:margins} is linear and $\beta>=1$). Then, the concavity of function $\varphi_\beta$ follows because for any $t \in [0,1]$, we have
\begin{align*}
\small
\varphi_\beta(x,t\B{\mu}_1+(1-t)\B{\mu}_2)\geq & \; t\nu_1+(1-t)\nu_2 \\ \geq &\;t\varphi_\beta(x,\B{\mu}_1)+(1-t)\varphi_\beta(x,\B{\mu}_2)
\end{align*}
as the maximum in the definition of \mbox{$\varphi_\beta(x,t\B{\mu}_1+(1-t)\B{\mu}_2)$} is not lower than the value of any feasible solution.

The convex optimization problem in \eqref{eq:mrc} is a bilevel optimization problem due to the definition of function $\varphi_\beta(x,\B{\mu})$. For $\beta \in (1,\infty)$, a closed form expression of the function cannot be attained while for the special cases $\beta=1$ and $\beta=\infty$ is given by \cite{MazSheYua:22}. However, in Section~\ref{sec:optimization} we show  that efficient learning can be achieved using stochastic gradients obtained by implicit differentiation.

\textbf{The resulting classification rule} $\up{h}_\beta$ corresponding to the MGCE method is given by the solution $\B{\mu}^*$ of \eqref{eq:mrc}. Specifically, for each instance $x \in \set{X}$ such rule assigns the probability corresponding to label $y \in \set{Y}$ as
\begin{equation}
\label{eq:soft_clf_mrc}
    \up{h}_\beta(x)_y = \Big(\frac{f(x, \B{\mu}^*)_y + \varphi_\beta(x, \B{\mu}^*)}{\beta} + 1\Big)^\beta_{+},
\end{equation}
where such an expression provides valid probabilities due to the definition of $\varphi_\beta$ in \eqref{eq:optimality_condition}.

The proposed MGCE formulation can also be viewed as obtaining classification rule $\up{h}_\beta$ from the classifier margins using the link function in \eqref{eq:soft_clf_mrc}. This link is tailored to the loss parameter $\beta$, whereas in existing methods the softmax function is used as a fixed link, independent of $\beta$. For $\beta \rightarrow \infty$, the link function in \eqref{eq:soft_clf_mrc} also reduces to the softmax function (see Appendix~\ref{sec:softmax_convg}), and in this case the loss coincides with CE.

The minimax formulation described above enables several theoretical guarantees. In the following, we establish a relation between the minimax probabilities and the worst-case distribution within the uncertainty set. Furthermore, we show that the proposed formulation not only bounds the classification risk associated with $\alpha$-losses, but also provides an upper bound on the classification error of the MGCE classification rule.

\subsection{Relation between worst-case distributions and minimax classifier}
\label{sec:worst-case-dist}
In this section, we characterize the worst-case distributions corresponding with the proposed MGCE formulation.
The following theorem presents the general relation between the worst-case distribution and the minimax probabilities for any loss parameter $\beta$.

\begin{theorem}
\label{th:minimax_worst_case}
Given a loss function $\ell_\beta$, if $\up{h}_\beta$ is the minimax classifier in \eqref{eq:soft_clf_mrc}, the worst-case distribution $\up{p}_\beta \in \arg \underset{\up{p} \in \set{U}}{\max} \ \mathbb{E}_{\up{p}}\ell_\beta(\up{h}_\beta,(x,y))$ is given by
\begin{equation}
\label{eq:worst_case_probs}
    \up{p}_\beta(y|x) = \frac{\up{h}_\beta(x)_y^\frac{\beta-1}{\beta}}{\sum_{y=1}^k \up{h}_\beta(x)_y^\frac{\beta-1}{\beta}}, \ \ \forall x\in \set{X}, \ y \in \set{Y}.
\end{equation}
Reciprocally, if $\up{p}_\beta$ is the worst-case distribution corresponding to the minimax problem in \eqref{eq:minmaxrisk}, that is, $ \up{p}_\beta \in \arg \underset{\up{p} \in \mathcal{U}}{\max} \, \underset{\up{h}}{\min}\; \mathbb{E}_{\up{p}}\ell_\beta(\up{h},(x,y))$, the minimax classifier in \eqref{eq:soft_clf_mrc} satisfies $\up{h}_\beta \in \arg  \underset{\up{h}}{\min}\; \mathbb{E}_{\up{p}_\beta}\ell_\beta(\up{h},(x,y))$ and is given by
\begin{equation}
\label{eq:minmax_probs}
    \up{h}_\beta(x)_y = \frac{\up{p}_\beta(y|x)^\frac{\beta}{\beta-1}}{\sum_{y=1}^k \up{p}_\beta(y|x)^\frac{\beta}{\beta-1}}.
\end{equation}
\end{theorem}
\begin{proof}
    See Appendix~\ref{sec:proof_th1}.
\end{proof}
Similar relationships between probabilities and classification rules have been observed in other works for composite losses \citep{bao2025calm}. Theorem~\ref{th:minimax_worst_case} characterizes the relation between the worst-case distribution $\up{p}_\beta$ and the classifier probabilities $\up{h}_\beta$ corresponding with different values of $\beta$. In particular, the worst-case distribution for a given solution $\B{\mu}^*$ corresponding with the minimax formulation \eqref{eq:mrc} is obtained as
\begin{equation}
\label{eq:worst_case_sol}
\resizebox{0.95\columnwidth}{!}{$\displaystyle
\up{p}_\beta(y|x) = \frac{\bigg(\Big(f(x,\B{\mu}^*)_y + \varphi_\beta(x, \B{\mu}^*)\Big) / \beta + 1\bigg)_+^{\beta-1}}{\sum_{j=1}^k\bigg(\Big(f(x,\B{\mu}^*)_j + \varphi_\beta(x, \B{\mu}^*)\Big)/\beta + 1\bigg)_+^{\beta - 1}}.$}
\end{equation}

The results in Theorem~\ref{th:minimax_worst_case} show how the relationship between worst-case distribution and classification rule changes with the values of $\beta$. Specifically, for $\beta \in (1, \infty)$, the worst-case probability satisfies $\up{p}_\beta(x)_y < \up{h}_\beta(x)_y$ for the high confidence classes $y$, that is, those for which $\up{h}_\beta(x)_y > \frac{1}{|\set{Y}|}$. On the other hand, the worst-case probability satisfies $\up{p}_\beta(x)_y \geq \up{h}_\beta(x)_y$ for the low confidence classes $y$, that is, those for which $\up{h}_\beta(x)_y \leq \frac{1}{|\set{Y}|}$ (see also Figure~\ref{fig:pu_hu_beta} corresponding with 2 classes). This behavior ensures that the minimax formulation penalizes overconfident predictions while preserving uncertainty, thereby performing a form of distributional smoothing on the classifier’s output. For the extreme cases \ac{MAE} $(\beta=1)$ and CE $(\beta=\infty)$, the worst-case distribution shows a different behavior. 
In case of \ac{MAE}, the worst-case distribution uniformly weighs all classes for which the classifier's probability is non-zero, irrespectively of its value (an extremely cautious behavior). However, when $\up{h}_\beta(x)_y=0$ or $\up{h}_\beta(x)_y=1$, the worst-case distribution also agrees with the classifier.
On the other hand, in case of CE, the worst-case distribution coincides with the classification probabilities, that is, $\up{h}_{\up{CE}}(x)_y = \up{p}_{\up{CE}}(x)_y$.

The relationship established above between the classifier probabilities and the worst-case distribution elucidates the optimization dynamics of the proposed MGCE methods. In Section~\ref{sec:optimization} we will show that the stochastic gradients used at learning are given by the worst-case distribution.

The MGCE formulation can also provide an upper bound on the \ac{MAE} classification risk of the classifier as shown in the following section. 

\begin{figure}
    \centering
    \psfrag{12345678912345}[l][l][0.9]{CE}
    \psfrag{12345678912346}[l][l][0.9]{$\beta = 2.0$}
    \psfrag{12345678912347}[l][l][0.9]{$\beta=1.4$}
    \psfrag{12345678912348}[l][l][0.9]{$\beta=1.1$}
    \psfrag{12345678912349}[l][l][0.9]{\ac{MAE}}
    \psfrag{x}[c][t][0.9]{$\up{h}_\beta(x)_y$}
    \psfrag{y}[c][t][0.9]{$\up{p}_\beta(y|x)$}
    \psfrag{loss}[c][t][0.9]{Loss}
    \psfrag{0.2}[r][r][0.8]{0.2}
    \psfrag{0.4}[r][r][0.8]{0.4}
    \psfrag{0.6}[r][r][0.8]{0.6}
    \psfrag{0.8}[r][r][0.8]{0.8}
    \psfrag{1}[r][r][0.8]{1}
    \psfrag{a}[c][c][0.8]{0.2}
    \psfrag{b}[c][c][0.8]{0.4}
    \psfrag{c}[c][c][0.8]{0.6}
    \psfrag{d}[c][c][0.8]{0.8}
    \psfrag{e}[c][c][0.8]{1}
    \psfrag{0}[t][b][0.8]{0}
    \psfrag{2}[r][r][0.8]{2}
    \psfrag{4}[r][r][0.8]{4}
    \psfrag{6}[r][r][0.8]{6}
    \psfrag{8}[r][r][0.8]{8}
    \psfrag{10}[r][r][0.8]{10}
    \includegraphics[width=0.45\textwidth]{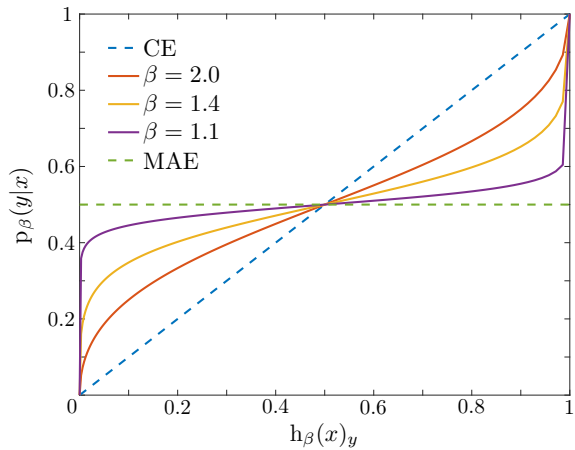}
    % \captionsetup{labelfont={it}, font=small}
    \caption{Relation between the minimax classifier $\up{h}_\beta(x)_y$ and the worst-case probability $\up{p}_\beta(y|x)$ corresponding with 2 classes. For $\beta \in (1,\infty)$, the worst-case probabilities take a cautious stance, avoiding the extremes of \ac{MAE} ($\beta=1$) and CE ($\beta=\infty$) losses. }
    \label{fig:pu_hu_beta}
    \vspace{-0.3cm}
\end{figure}
% \vspace{-0.2cm}
\subsection{Performance guarantees}
\vspace{-0.1cm}
The proposed MGCE formulation can serve as a surrogate for the \ac{MAE} loss. In the following, we characterize the closeness of the MGCE to the \ac{MAE} loss in terms of the loss parameter $\beta$. In addition, we show that our proposed formulation can provide an upper bound on the classification error.

\begin{theorem}
\label{th:tight_bound}
    The probabilistic $\alpha$-losses in \eqref{eq:loss_beta} satisfy
    \begin{equation}
        \label{eq:lb_ub_alpha_loss}
        \frac{1}{\beta}\ell_{\beta}(\up{h}, (x, y)) \leq \ell_{\mathrm{MAE}}(\up{h}, (x, y)) \leq \ell_{\beta}(\up{h}, (x, y)).
    \end{equation}
    Moreover, for $\beta \in (1,\infty)$, we have
    \begin{equation}
        \label{eq:ub_alpha_loss}
        \ell_{\beta}(\up{h}, (x, y)) - \ell_{\mathrm{MAE}}(\up{h}, (x, y)) \leq \beta - 1.
    \end{equation}
\end{theorem}
\begin{proof}
    See Appendix~\ref{sec:proof_theorem2}.
\end{proof}

Theorem~\ref{th:tight_bound} shows that the $\alpha$-losses can provide upper bounds for \ac{MAE} loss. 
Using this relation, we also obtain performance guarantees in terms of an upper bound on the \ac{MAE} classification risk $\mathcal{R}_{\mathrm{MAE}}(\up{h})$ for the minimax classifier $\up{h}_\beta$ derived from the MGCE formulation.

\begin{remark}
\label{rm:tight_bound}
As a consequence of inequality \eqref{eq:lb_ub_alpha_loss}, $\mathcal{R}_{\mathrm{MAE}}({\up{h}_{\beta}}) \leq \mathcal{R}_{\beta}({\up{h}_{\beta}})$ for a minimax classification rule $\up{h}_{\beta}$ corresponding with \eqref{eq:mrc}. Moreover, using inequality \eqref{eq:upper_bound_alpha_risk}, we have that 
\begin{equation}
\label{eq:upper_bound_0_1_risk}
    \mathcal{R}_{\mathrm{MAE}}(\up{h}_{\beta}) \leq \up{V}_{\beta}
\end{equation}
in cases where the underlying distribution $\up{p}^*$ lies in the uncertainty set $\set{U}$ in \eqref{eq:us}.
\end{remark}

The remark above shows that the optimal value $\up{V}_{\beta}$ obtained by the proposed MGCE formulation in \eqref{eq:mrc} can provide an upper bound on the classification error. Notice that the error probability of a classification rule coincides with the \ac{MAE} risk for randomized rules and is at most twice the \ac{MAE} risk for deterministic rules \citep{MazRomGrun:23}. Moreover, inequality \eqref{eq:ub_alpha_loss} in Theorem~\ref{th:tight_bound} shows that the upper bound in \eqref{eq:upper_bound_0_1_risk} is tighter as $\beta$ decreases.

As described above, the proposed MGCE provides performance guarantees together with a convex optimization over margins. However, the bilevel optimization in (8) prevents the usage of conventional methods for large-scale learning. In the following, we present an efficient learning algorithm for MGCE methods based on implicit differentiation.

\vspace{-0.1cm}
\section{OPTIMIZATION}
\vspace{-0.2cm}
\label{sec:optimization}
In this section, we detail the training procedure for the convex optimization problem of MGCE defined in \eqref{eq:mrc}. Such optimization is amenable to \ac{SGD} as the objective is given in terms of expectations. The main technical difficulty lies in the fact that the function $\varphi_\beta(x, \B{\mu})$ in \eqref{eq:implicit_function} is only implicitly defined as in \eqref{eq:optimality_condition}. In the following, we provide an efficient algorithm to compute the stochastic gradients in \eqref{eq:stoc_grad} using implicit differentiation \citep{dontchev2009implicit, agrawal2019differentiable}.

\vspace{-0.1cm}
\subsection{Stochastic gradients}
\vspace{-0.2cm}
A stochastic gradient $g_{\beta}(x,y)$ of the objective in \eqref{eq:mrc} for training sample $(x,y)$ is given by 
\begin{align}
    \label{eq:subgrad_beta}
g_{\beta}(x,y) = \B{\lambda} \odot \text{sign}(\B{\mu})-\Phi(x,y) - \frac{\partial\varphi_\beta(x, \B{\mu})}{\partial \B{\mu}},
\end{align}
where $\odot$ denotes the Hadamard product, sign$(\B{\mu})$ denotes the vector given by the signs of the components of vector $\B{\mu}$. Such gradient is directly derived from the expression $\B{\tau}$ in \eqref{eq:tau}, but requires to evaluate the derivative of function $\varphi_\beta(x,\B{\mu})$ that does not have a closed-form expression.

The following theorem derives the gradient of function $\varphi_\beta(x,\B{\mu})$ using implicit differentiation. Interestingly, the gradient is determined by the worst-case distribution of the classifier corresponding with $\B{\mu}$.
\begin{theorem}
    \label{th:stoc_gradient}
    The stochastic gradient of the function $\varphi(x,\B{\mu})$ corresponding with sample $x$ is given as 
    \begin{equation}
        \label{eq:stoc_grad}
        \frac{\partial\varphi_\beta(x, \B{\mu})}{\partial \B{\mu}} = -\sum_{y=1}^k\up{p}_\beta(y|x)\Phi(x,y),
    \end{equation}
    where $\up{p}_\beta(y|x)$ is the worst-case distribution corresponding to the parameters $\B{\mu}$ and sample $x$ given in \eqref{eq:worst_case_sol}.
\end{theorem}
\begin{proof}
    See Appendix~\ref{sec:proof_theorem3}.
\end{proof}

Theorem~\ref{th:stoc_gradient} shows that the gradient used by the proposed MGCE formulation varies as a function of the loss parameter $\beta$. In particular, the gradient of $\varphi_{\beta}$ corresponds to an expectation with respect to the worst-case distribution $\up{p}_\beta$.

The results above show that the proposed MGCE formulation is amenable to \ac{SGD}, similarly as methods based on margin losses. Indeed, the proposed approach corresponds with the margin loss defined as
\begin{equation}
    \label{eq:margin_loss}
    \ell_\beta(\B{\mu},(x,y))=-f(x,\B{\mu})_{y}-\varphi_\beta(x,\B{\mu}),
\end{equation}
which is convex on $\B{\mu}$.
Furthermore, the following corollary shows that such MGCE margin loss is classification-calibrated.

\begin{corollary}
\label{cor:bayes_consistency}
Let $\ell_\beta$ be the margin loss defined in \eqref{eq:margin_loss}. For any $x$ and $\beta>1$, if 
\begin{equation}
    \B{\mu}^*\in\arg\min_{\B{\mu}} \sum_y \up{p}^*(y|x) \ell_\beta(\B{\mu},(x,y))
\end{equation}
then
\begin{equation}
    \arg\max_y f(x,\B{\mu}^*)_y=\arg\max_y \up{p}^*(y|x).
\end{equation}
\end{corollary}
\begin{proof}
    See Appendix~\ref{sec:proof_corollary_1}.
\end{proof}
Beyond classification-calibration, the composite loss corresponding with MGCE leads to a proper loss since MGCE losses are ``calm'' \citep{bao2025calm}. Specifically, a composite loss for probabilities $\tilde\ell_\beta(\up{p}, (x,y))$ can be defined using the margin loss in \eqref{eq:margin_loss} as $\tilde\ell_\beta(\up{p}, (x, y)) = \ell_\beta (\B{\mu}_{\up{p}}, (x, y))$, where $\B{\mu}_{\up{p}}$ is the parameter vector associated with probability $\up{p}$ via \eqref{eq:worst_case_sol}. Then, an argument similar to that in the proof of Corollary~\ref{cor:bayes_consistency} shows that the loss $\tilde{\ell}_\beta$ is proper for any $\beta > 1$.

We can establish a connection between the classifier outputs and the gradients obtained during the \ac{SGD} iterations for $\beta \in [1,\infty)$ using the relation between the worst-case distribution and the classifier outputs in Section~\ref{sec:worst-case-dist}. Relative to the \ac{MAE} loss ($\beta=1$), the gradient of MGCE associated with larger values of $\beta$ places more emphasis on the high confidence classes, i.e., the labels with classifier probability exceeding $\tfrac{1}{|\set{Y}|}$ at the current iteration. Conversely, relative to the CE loss ($\beta=\infty$), the gradient of MGCE for smaller values of $\beta$ places less emphasis on high confidence classes, while allowing for more uncertainty by emphasizing the low confidence classes. As a result, for clean samples, larger $\beta$ values accelerate learning from confident and correctly classified examples. In contrast, for noisy samples, smaller $\beta$ values mitigate the influence of confident but potentially mislabeled examples.

\vspace{-0.1cm}
\subsection{Effective bisection method}
\vspace{-0.2cm}
The worst-case distribution $\up{p}_\beta(y|x)$ is defined in terms of the implicit function $\varphi_\beta(x, \B{\mu})$, so that, the gradient in \eqref{eq:stoc_grad} requires evaluation of $\varphi_\beta(x, \B{\mu})$. For a given $\B{\mu}$ and $x$, this value can be computed by solving the equation $F(x, \B{\mu}, \varphi^i_\beta)=1$ in \eqref{eq:optimality_condition} 
for $\varphi^i_\beta$ using the bisection method as detailed below.

A unique root for equation \eqref{eq:optimality_condition} exists since $F(x, \B{\mu}, \varphi^i_\beta)$ is continuous and monotonically increasing in $\varphi^i_\beta$. In addition, such a root is included in the interval $[\varphi_\beta^{1}, \varphi_\beta^{2}]$ given by 
\begin{equation}
    \begin{array}{cc}
         \varphi_\beta^{1}=&  C_\beta- \max_{i\in\set{Y}} f(x,\B{\mu})_i,\\
         \varphi_\beta^{2}=& C_\beta - \min_{i\in\set{Y}} f(x,\B{\mu})_i,
    \end{array}\label{eq:varphi_limits}
\end{equation}
where $C_\beta = \beta(\tfrac{1}{k^{1/\beta}}-1)$. This interval contains a root because \mbox{$F(x,\B{\mu}, \varphi^{1}_\beta) \leq 1$} and $F(x,\B{\mu}, \varphi^{2}_\beta) \geq 1$. Specifically, these inequalities follow because for each $i \in \set{Y}$, we have 
\begin{equation}
    \nonumber
    \mbox{$\Big(\frac{f(x, \B{\mu})_i + \varphi_\beta^1}{\beta} + 1\Big)^{\beta}\leq \frac{1}{|\set{Y}|}$} 
\end{equation}
and
\begin{equation}
    \nonumber
    \mbox{$\Big(\frac{f(x, \B{\mu})_i + \varphi_\beta^2}{\beta} + 1\Big)^{\beta}\geq \frac{1}{|\set{Y}|}$}.
\end{equation}
Therefore, $\varphi_\beta(x, \B{\mu})$ can be efficiently evaluated by applying the bisection method over such an interval.

% {\linespread{1.3}\selectfont
\begin{algorithm}[t]
\setlength{\textfloatsep}{0pt}
\captionsetup{labelfont={bf}, format=hang}
\caption{Implicit gradient for MGCE \\ via bisection.}
\label{alg:learning_alg}
\renewcommand{\arraystretch}{1.1}
\begin{tabular}{ll}
\textbf{Input:} & \hspace{-0.4cm} Instance $x$, model parameters $\B{\mu}$, \\
& \hspace{-0.4cm} loss parameter $\beta$, tolerance $\epsilon$ \\
\textbf{Output:} & \hspace{-0.4cm} Gradient $\frac{\partial\varphi_\beta(x, \B{\mu})}{\partial \B{\mu}}$ \\
\end{tabular}
\begin{algorithmic}[1]
\setstretch{1.2}
\State Compute $\varphi_{\beta}^1$ and $\varphi_{\beta}^2$ as in \eqref{eq:varphi_limits}
\While{$|\varphi_\beta^{1} - \varphi_\beta^{2}| \ge \epsilon$}
    \State $\varphi_\beta^i \gets (\varphi_\beta^{1} + \varphi_\beta^{2})/2$
    \If{$F(\B{x, \B{\mu}}, \varphi_\beta^i) > 1$}
        \State $\varphi_\beta^{2} \gets \varphi_\beta^i$
    \Else
        \State $\varphi_\beta^{1} \gets \varphi_\beta^i$
    \EndIf
\EndWhile
\State $\varphi_\beta(x,\B{\mu}) \gets \varphi_\beta^{i}$
\State Compute $\up{p}_\beta(y|x)$ using $\varphi_\beta(x,\B{\mu})$ as in \eqref{eq:worst_case_sol}
\State $\frac{\partial\varphi_\beta(x, \B{\mu})}{\partial \B{\mu}} \gets -\sum_{y=1}^k\up{p}_\beta(y|x)\Phi(x,y)$
\end{algorithmic}
\end{algorithm}
% }

Algorithm~\ref{alg:learning_alg} details the gradient computation step for \ac{SGD}, which has a computational complexity of \mbox{$O\Big(\log\Big((\varphi_{\beta}^2 -\varphi_{\beta}^1)/\epsilon\Big)\Big)$} for $\epsilon$ level of precision in the bisection routine. In the following, we evaluate the performance of the proposed MGCE formulation using the efficient learning Algorithm~\ref{alg:learning_alg} on \acp{DNN}.

\begin{figure*}
  \centering
  \begin{subfigure}[b]{0.32\textwidth}
    \centering
    \psfrag{1234567891234567891234}[l][l][0.9]{MGCE}
     \psfrag{1234567891234567891235}[l][l][0.9]{GCE}
     \psfrag{y}[c][c][0.9]{Test accuracy (\%)}
     \psfrag{Epoch}[c][c][0.9]{Epochs}
    \psfrag{loss}[c][t][1]{Loss}
    \psfrag{50}[r][r][0.8]{50}
    \psfrag{0}[r][b][0.8]{0}
    \psfrag{60}[r][r][0.8]{}
    \psfrag{80}[r][r][0.8]{}
    \psfrag{150}[c][b][0.8]{150}
    \psfrag{a}[c][b][0.8]{50}
    \psfrag{70}[r][r][0.8]{70}
    \psfrag{90}[r][r][0.8]{90}
    \psfrag{100}[c][b][0.8]{100}
    \includegraphics[width=\linewidth]{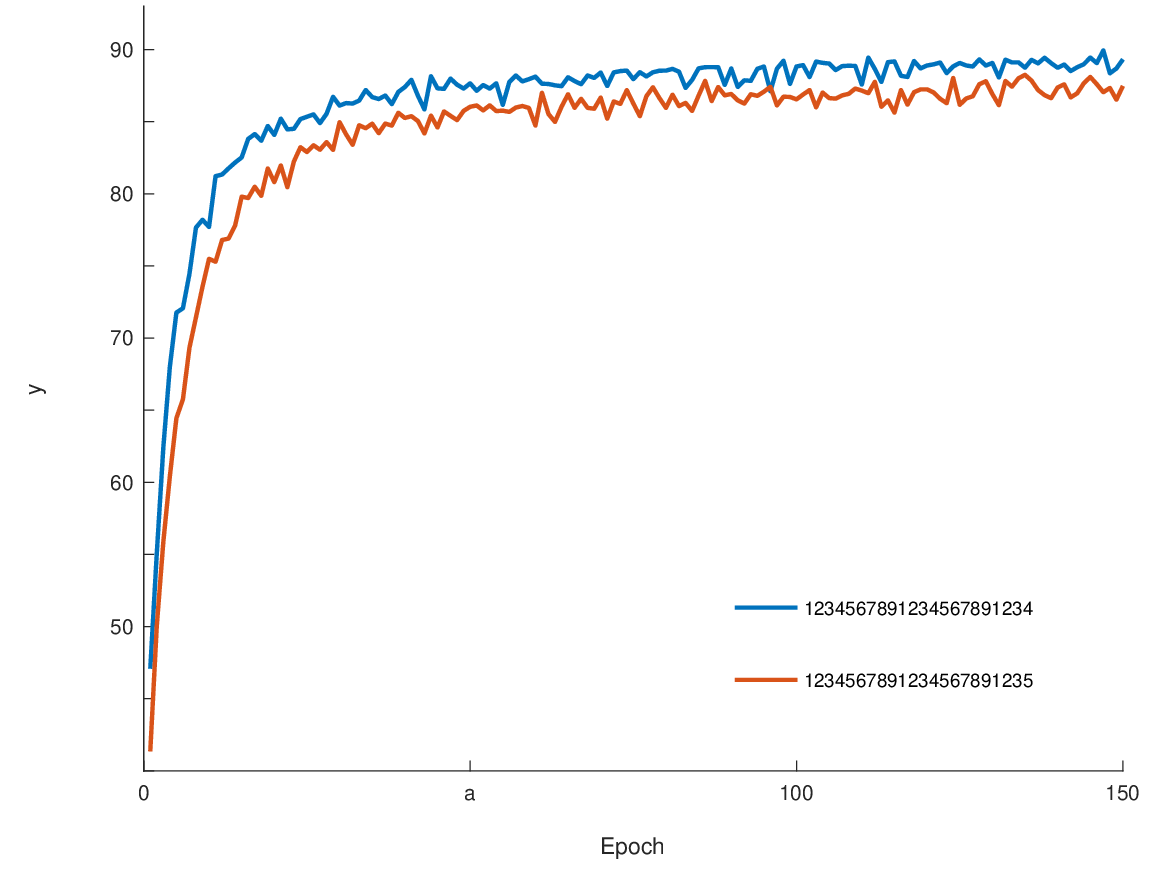} % replace with your file
    \caption{CIFAR10}
    \label{fig:cifar10}
  \end{subfigure}
  \hfill
  \begin{subfigure}[b]{0.32\textwidth}
    \centering
    \psfrag{1234567891234567891234}[l][l][0.9]{MGCE}
     \psfrag{1234567891234567891235}[l][l][0.9]{GCE}
     \psfrag{Epoch}[c][c][0.9]{Epochs}
    \psfrag{loss}[c][t][0.8]{Loss}
    \psfrag{0}[r][b][0.8]{0}
    \psfrag{150}[c][b][0.8]{150}
    \psfrag{10}[r][r][0.8]{}
    \psfrag{30}[r][r][0.8]{}
    \psfrag{100}[c][b][0.8]{100}
    \psfrag{50}[r][t][0.8]{}
    \psfrag{60}[r][r][0.8]{60}
    \psfrag{20}[r][r][0.8]{20}
    \psfrag{40}[r][r][0.8]{40}
    \psfrag{60}[r][r][0.8]{60}
    \psfrag{b}[c][b][0.8]{50}
    \includegraphics[width=\linewidth]{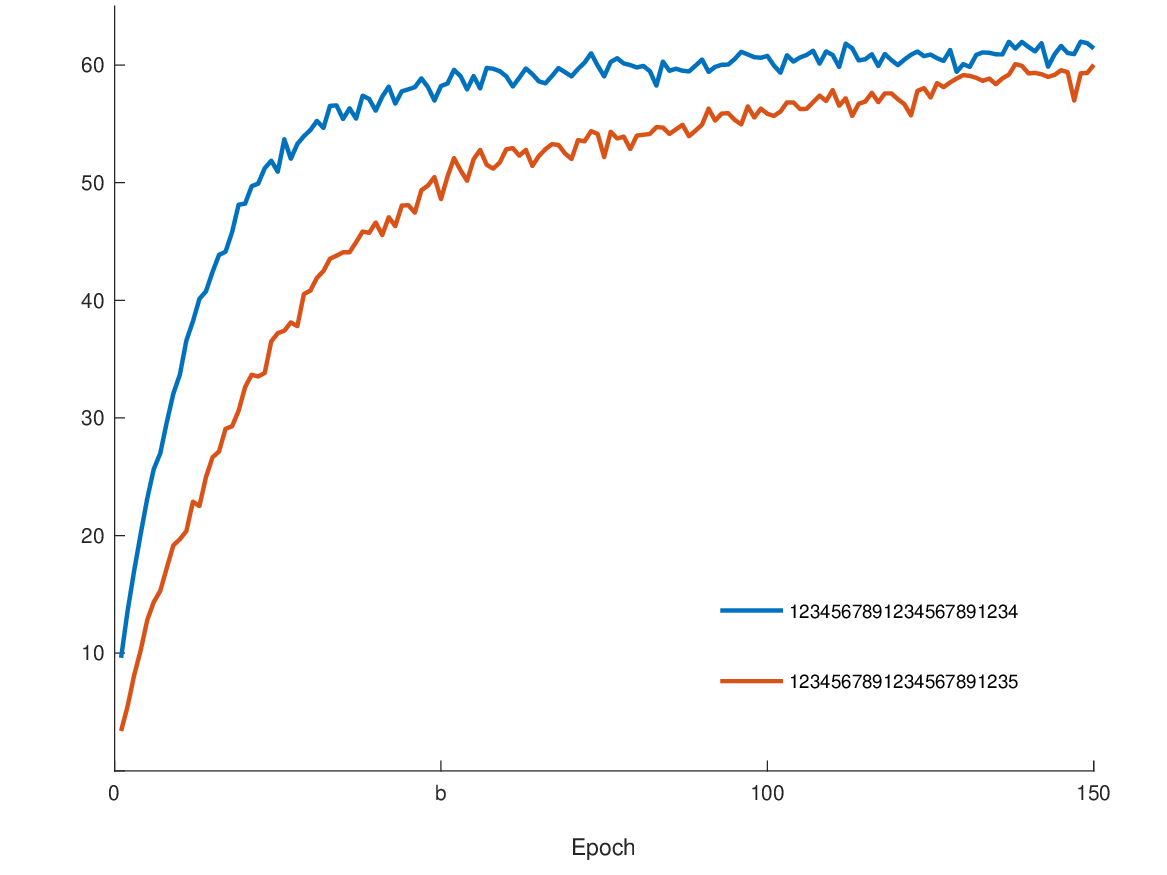} % replace with your file
    \caption{CIFAR100}
    \label{fig:cifar100}
  \end{subfigure}
  \hfill
  \begin{subfigure}[b]{0.32\textwidth}
    \centering
    \psfrag{1234567891234567891234}[l][l][0.9]{MGCE}
     \psfrag{1234567891234567891235}[l][l][0.9]{GCE}
     \psfrag{Epoch}[c][c][0.9]{Epochs}
    \psfrag{loss}[c][t][0.8]{Loss}
    \psfrag{0}[t][b][0.8]{0}
    \psfrag{150}[c][b][0.8]{150}
    \psfrag{100}[c][b][0.8]{100}
    \psfrag{50}[c][b][0.8]{50}
    \psfrag{10}[r][r][0.8]{10}
    \psfrag{20}[r][r][0.8]{20}
    \psfrag{30}[r][r][0.8]{30}
    \psfrag{40}[r][r][0.8]{40}
    \psfrag{60}[r][r][0.8]{60}
    \includegraphics[width=\linewidth]{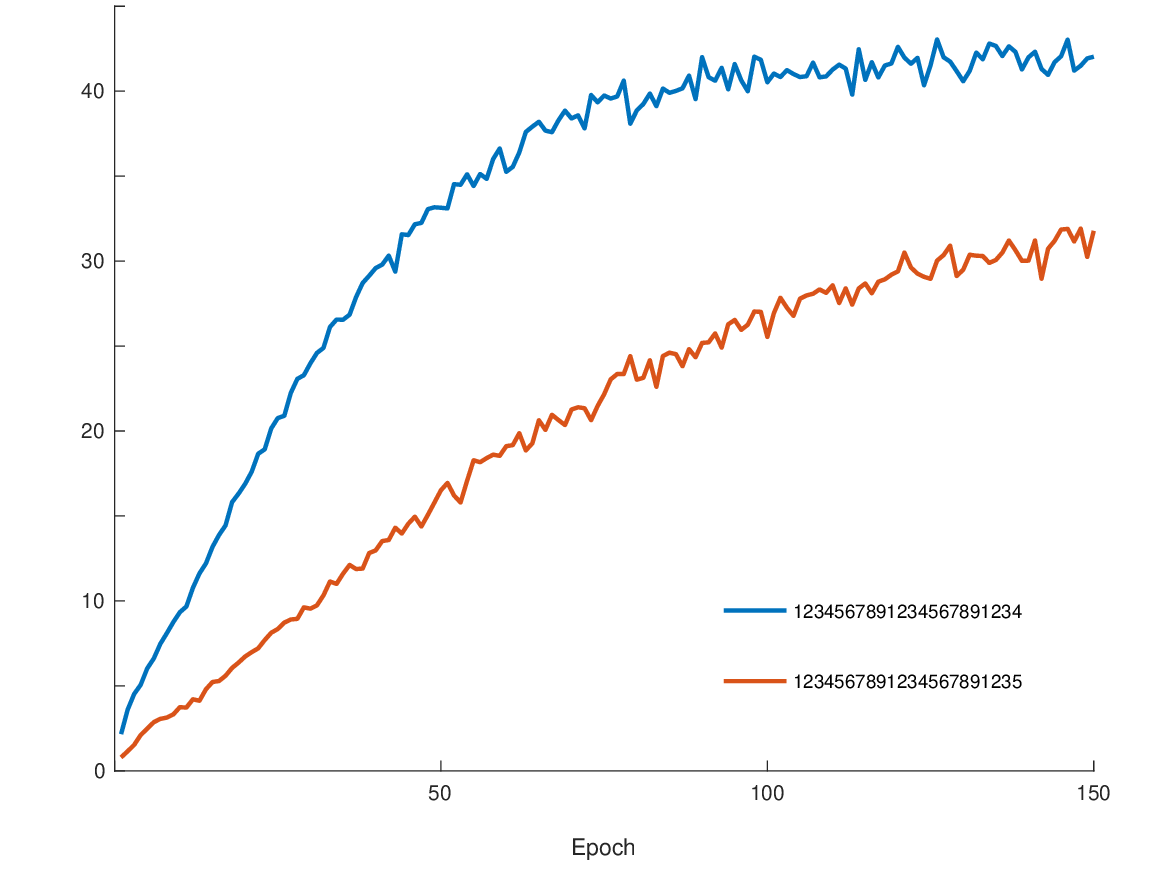} % replace with your file
    \caption{Tiny ImageNet}
    \label{fig:imagenet}
  \end{subfigure}
  \caption{Average test accuracy under clean training data obtained for multiple complex datasets. The value of loss parameter $\beta$ is set to 1.4. The figure shows the fast convergence of the proposed MGCE in comparison to the GCE.}
  \label{fig:convergence}
  \vspace{-0.15cm}
\end{figure*}

%\vspace{-0.1cm}
\section{EXPERIMENTAL RESULTS}
\vspace{-0.2cm}
In this section, we present numerical results evaluating the performance of the 
proposed MGCE formulation in terms of test accuracy, iteration complexity, and 
calibration error using \acp{DNN} corresponding with ResNet architectures \citep{he2016deep}. The experimental results are carried out using the datasets
\mbox{FashionMNIST}, \mbox{CIFAR-10}, SVHN, \mbox{CIFAR-100}, and \mbox{Tiny ImageNet} that are commonly used as benchmarks. In addition, we also assess the performance on the large benchmark dataset WebVision affected by real-world label noise \citep{li2017webvisiondatabasevisuallearning}. For \mbox{CIFAR-10} and \mbox{CIFAR-100}, we employ the standard ResNet-34 architecture, while for \mbox{FashionMNIST}, \mbox{Tiny ImageNet}, and SVHN, we use ResNet-18 architecture. For the WebVision dataset, we use the standard ResNet-50 architecture.

We compare the proposed MGCE method with the GCE method as presented in \cite{zhang2018generalized}. Additionally, we include comparisons with the standard CE loss 
and the convex \ac{MAE} loss based on the minimax approach proposed in \cite{MazRomGrun:23}. 
Additional results are presented in Appendix~\ref{sec:add_results}. The implementation of the proposed MGCE formulation is available in the python library MRCpy~\citep{BonMazPer:24}.\footnote{\url{https://github.com/MachineLearningBCAM/MRCpy/tree/main/MRCpy/pytorch/mgce}}

\vspace{-0.1cm}
\subsection{Experimental setup}
\vspace{-0.2cm}
The following setup applies to all the experiments conducted. As part of data preprocessing and augmentation procedure, we apply per-pixel mean subtraction, horizontal random flipping, and random cropping to 32×32 after padding the images with 4 pixels on each side. For training, we use \ac{SGD} with a momentum of 0.9 and a fixed learning rate of 0.01. We choose the regularization parameter $\lambda_0 = 10^{-5}$ based on grid search among common values as detailed in Appendix~\ref{sec:regularization_param_search}. Algorithm~\ref{alg:learning_alg} with tolerance $\epsilon=10^{-4}$ is used to compute the gradients for the proposed MGCE. For all settings, we clip the gradient norm to 5.0. We train the entire network for 150 epochs using mini-batch of size 128, and we use 10\% of the entire training data for validation. For the experiments corresponding with noisy settings, both training and validation set are contaminated with symmetric label noise, and all the experiments are conducted for five times with randomization. 
\begin{table*}[ht]
\centering
\setlength{\tabcolsep}{3pt} % Adjust horizontal spacing
\renewcommand{\arraystretch}{1.2} % Adjust row height
\resizebox{\textwidth}{!}{\begin{tabular}{c|c|c|c|c|c|c|c}
\hline
\multirow{2}{*}{Dataset} & \multirow{2}{*}{Method}
& \multicolumn3{c|}{Test Accuracy (\%)} 
& \multicolumn3{c}{SCE (\%)} \\
\cline{3-8} 
& 
& $\eta$=0.0 & $\eta$=0.2 & $\eta$=0.4 & $\eta$=0.0 & $\eta$=0.2 & $\eta$=0.4 \\
\hline
        \multirow{4}{*}{FashionMNIST} & MGCE
             & \textbf{92.94 $\pm$ 0.17}
             & 90.11 $\pm$ 0.38
             & 88.65 $\pm$ 0.15
            & \textbf{05.20 $\pm$ 0.23}
            & \textbf{02.44 $\pm$ 0.28}
            & \textbf{02.54 $\pm$ 1.30}
        \\
         & GCE
             & 92.93 $\pm$ 0.34
             & \textbf{91.11 $\pm$ 0.22}
             & \textbf{89.74 $\pm$ 0.29}
            & 05.52 $\pm$ 0.31%\textbf{05.52 $\pm$ 0.31}
            & 07.14 $\pm$ 0.22
            & 08.63 $\pm$ 0.35
        \\
         & CE
             & 92.73 $\pm$ 0.32
             & 89.62 $\pm$ 0.15
             & 87.40 $\pm$ 0.13
            & 05.33 $\pm$ 0.24
            & 17.32 $\pm$ 0.85
            & 33.52 $\pm$ 1.86
        \\
         & \ac{MAE}
             & 91.84 $\pm$ 0.24
             & 89.86 $\pm$ 0.26
             & 88.44 $\pm$ 0.38
            & 73.47 $\pm$ 0.57
            & 74.88 $\pm$ 0.46
            & 74.23 $\pm$ 0.28
        \\
    \hline
        \multirow{4}{*}{CIFAR10} & MGCE
             & \textbf{91.01 $\pm$ 0.43}
             & 86.91 $\pm$ 0.36%\textbf{86.91 $\pm$ 0.36}
             & 83.49 $\pm$ 0.45%\textbf{83.49 $\pm$ 0.45}
            & \textbf{05.53 $\pm$ 0.31}
            & \textbf{10.26 $\pm$ 0.37}
            & \textbf{12.15 $\pm$ 0.5}
        \\
         & GCE
             & 90.55 $\pm$ 0.25
             & \textbf{86.95 $\pm$ 0.43}
             & \textbf{83.80 $\pm$ 0.57}
            & 06.14 $\pm$ 0.13
            & 10.88 $\pm$ 0.48%\textbf{10.88 $\pm$ 0.48}
            & 12.95 $\pm$ 0.66
        \\
         & CE
             & 90.80 $\pm$ 0.27
             & 83.73 $\pm$ 0.37
             & 77.31 $\pm$ 0.93
            & 05.71 $\pm$ 0.44
            & 10.24 $\pm$ 1.86
            & 25.31 $\pm$ 2.40
        \\
         & \ac{MAE}
             & 89.31 $\pm$ 0.22
             & 86.01 $\pm$ 0.46
             & 82.64 $\pm$ 0.61
            & 73.00 $\pm$ 0.24
            & 71.54 $\pm$ 0.46
            & 68.29 $\pm$ 0.59
        \\
        \hline
                \multirow{4}{*}{SVHN} & MGCE
             & \textbf{95.04 $\pm$ 0.21}
             & \textbf{94.03 $\pm$ 0.31}
             & \textbf{92.92 $\pm$ 0.80}
            & 03.19 $\pm$ 0.33%\textbf{03.19 $\pm$ 0.33}
            & \textbf{04.30 $\pm$ 0.24}
            & \textbf{03.96 $\pm$ 0.60}
        \\
         & GCE
             & 94.68 $\pm$ 0.24
             & 93.43 $\pm$ 0.45
             & 92.84 $\pm$ 0.26%\textbf{92.84 $\pm$ 0.26}
            & \textbf{02.58 $\pm$ 0.27}
            & 12.27 $\pm$ 1.75
            & 05.27 $\pm$ 0.27
        \\
         & CE
             & 93.49 $\pm$ 1.23
             & 93.15 $\pm$ 0.25
             & 90.88 $\pm$ 0.94
            & 04.22 $\pm$ 1.57
            & 15.01 $\pm$ 1.07
            & 32.61 $\pm$ 4.63
        \\
         & \ac{MAE}
             & 92.89 $\pm$ 1.45
             & 93.71 $\pm$ 0.41
             & 92.32 $\pm$ 0.41
            & 75.76 $\pm$ 2.83
            & 79.15 $\pm$ 0.32
            & 77.91 $\pm$ 0.36
        \\
    \hline
        \multirow{4}{*}{CIFAR100} & MGCE
             & 64.35 $\pm$ 0.45%\textbf{64.35 $\pm$ 0.45}
             & \textbf{58.88 $\pm$ 0.76}
             & \textbf{52.65 $\pm$ 0.85}
            & \textbf{22.22 $\pm$ 0.09}
            & 24.50 $\pm$ 1.05
            & 23.54 $\pm$ 0.65
        \\
         & GCE
             & \textbf{64.82 $\pm$ 0.38}
             & 56.66 $\pm$ 0.85
             & 51.37 $\pm$ 0.68
            & 22.89 $\pm$ 0.34
            & 33.77 $\pm$ 0.65
            & 37.16 $\pm$ 0.81
        \\
         & CE
             & 64.66 $\pm$ 0.39
             & 52.51 $\pm$ 1.10
             & 42.32 $\pm$ 0.78
            & 22.19 $\pm$ 0.13
            & \textbf{04.30 $\pm$ 1.80}
            & \textbf{06.06 $\pm$ 2.24}
        \\
         & \ac{MAE}
             & 62.17 $\pm$ 0.59
             & 57.57 $\pm$ 0.49
             & 41.99 $\pm$ 1.45
            & 60.47 $\pm$ 0.59
            & 56.15 $\pm$ 0.49
            & 60.46 $\pm$ 0.59
        \\
    \hline
\end{tabular}}
\caption{Average test accuracy and SCE for clean and noisy data. The ratio $\eta$ represents the fraction of samples corrupted by symmetric label noise. We report accuracies of the epoch where validation accuracy is maximum, and the loss parameter $\beta$ is selected using cross-validation. Our proposed MGCE method obtains strong accuracy together with improved calibration.}
\label{tab:loss_comparison_DNN}
\vspace{-0.3cm}
\end{table*}

\begin{figure}  % [!t] tries to place figure at top
  \centering
\psfrag{1234567891234}[l][l][0.9]{MGCE}
     \psfrag{1234567891235}[l][l][0.9]{GCE}
     \psfrag{y}[c][c][0.9]{Accuracy (\%)}
     \psfrag{Epoch}[c][c][0.9]{Epochs}
    \psfrag{loss}[c][t][0.9]{Loss}
    \psfrag{0}[t][b][0.8]{0}
    \psfrag{150}[c][c][0.8]{150}
    \psfrag{250}[c][c][0.8]{250}
    \psfrag{45}[r][r][0.8]{}
    \psfrag{55}[r][r][0.8]{}
    \psfrag{60}[r][r][0.8]{}
    \psfrag{65}[r][r][0.8]{}
    \psfrag{70}[r][r][0.8]{}
    \psfrag{75}[r][r][0.8]{}
    \psfrag{80}[r][r][0.8]{}
    \psfrag{85}[r][r][0.8]{}
    \psfrag{10}[r][r][0.8]{}
    \psfrag{20}[r][r][0.8]{20}
    \psfrag{30}[r][r][0.8]{}
    \psfrag{40}[r][r][0.8]{40}
    \psfrag{50}[c][c][0.8]{50}
    \psfrag{60}[r][r][0.8]{60}
    % \psfrag{a}[r][r][0.8]{0}
    % \psfrag{b}[c][c][0.8]{50}
    % \psfrag{c}[c][r][0.8]{100}
    % \psfrag{d}[c][r][0.8]{150}
  % ========== Single figure spanning one column ==========
  \includegraphics[width=\columnwidth]{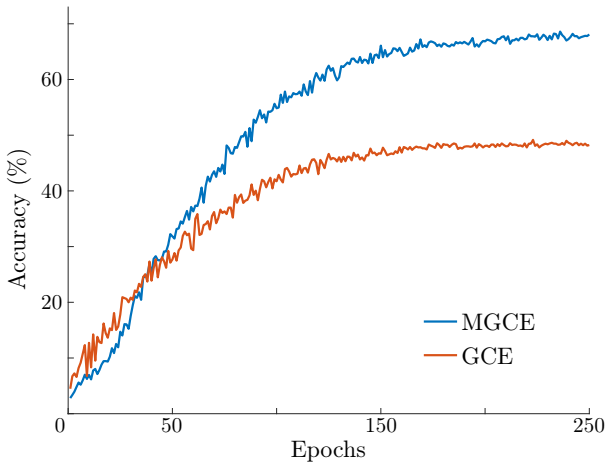}  % adjust filename
  \caption{Top-1 validation accuracy on the real-world noisy dataset WebVision. 
The figure shows that the proposed MGCE outperforms GCE, which significantly underfits on this complex dataset due to its non-convexity.}
\vspace{-0.5cm}
  \label{fig:webvision}
\end{figure}

\vspace{-0.6cm}
\subsection{Comparison between GCE and MGCE in terms of accuracy and convergence}
\vspace{-0.2cm}

We compare the proposed MGCE method with GCE using DNNs on multiple datasets CIFAR10 (10 classes), CIFAR100 (100 classes), and Tiny ImageNet (200 classes). Figure~\ref{fig:convergence} shows the test accuracy across training epochs for $\beta = 1.4$, which is the default value suggested in \cite{zhang2018generalized} for the GCE method.

The figure illustrates that the proposed MGCE converges to a better accuracy faster than the GCE method. 
Specifically, the difference is significant for complex datasets with many classes. For instance, in \mbox{CIFAR-100}, the MGCE method achieves the same accuracy as GCE at epoch 50 for a significantly lower number of epochs.

\vspace{-0.2cm}
\subsection{Evaluation on real-world noisy label}
\vspace{-0.1cm}
We further assess the performance of the proposed MGCE formulation using the real-world noisy dataset WebVision \citep{li2017webvisiondatabasevisuallearning} that contains more than 2.4 million web images crawled from the internet. Here, we follow the "Mini" setting proposed in \cite{jiang2018mentornet} that only takes the first 50 classes of the Google resized image subset as the training dataset. A \mbox{ResNet-50} architecture is trained under different methods and evaluated on the corresponding clean WebVision validation set.
% The trained network is evaluated on the same 50 classes of the WebVision validation set which is considered to be clean. 
We use the same experimental setup as in \cite{ma2020normalized,zhou2021asymmetric}, which is detailed in the Appendix~\ref{sec:real_world_setup}. 

Figure~\ref{fig:webvision} reports the top-1 validation accuracy on the clean validation set.
The results show that the proposed MGCE method significantly outperforms GCE using the same value $\beta=1.4$. Specifically, results demonstrate that the GCE method may underfit with such complex dataset inherent with real-world noise. 

\vspace{-0.2cm}
\subsection{Cross-validation over general values of \texorpdfstring{$\beta$}{beta}}
\vspace{-0.1cm}
We further evaluate the performance of MGCE and GCE across general values of the parameter $\beta$. In particular, we compare with the baselines provided by CE and the convex \ac{MAE} loss introduced in \cite{MazRomGrun:23}. We assess the performance in terms of the test accuracy and model calibration under both clean and noisy label conditions. For the noisy setting, we consider symmetric noise with rate $\eta=0.2$ and $\eta=0.4$ introduced synthetically to both the training and validation sets. We consider static calibration error (SCE) that takes into account the probabilities corresponding with all the classes rather just the maximum prediction as defined in \cite{nixon2019measuring}. The loss parameter $\beta$ is selected via cross-validation over 8 values in the interval $(1.05, 11)$. This range is sufficient in practice, as values below $1.05$ yield results similar to the \ac{MAE} loss, while values above 11 yield results similar to CE. 

The results in Table~\ref{tab:loss_comparison_DNN} show that the proposed MGCE method attains strong accuracy while yielding improved calibration, particularly in the presence of label noise. Without noise, MGCE and GCE obtain similar accuracy as that with CE. On the other hand, in noisy scenarios, both methods demonstrate an increased robustness surpassing the \ac{MAE} and CE baselines, especially on complex datasets such as CIFAR100.
Beyond accuracy, MGCE results in an enhanced model calibration, especially in noisy scenarios. Although CE achieves better calibration on CIFAR-100, it does so at the expense of accuracy.

\vspace{-0.1cm}
\section{CONCLUSION}
\vspace{-0.3cm}
In this paper, we introduce a minimax formulation for generalized cross-entropy (MGCE). While existing formulations of generalized cross-entropy lead to non-convex optimization over the classifier margins, our formulation casts learning as a convex bilevel optimization problem. Specifically, such formulation minimizes the worst-case expected \mbox{$\alpha$-loss} with respect to an uncertainty set of distributions. We show that the MGCE formulation can provide an upper bound on the classification error, and leads to classification-calibrated margin losses. Moreover, the paper establishes a relationship between the MGCE classifier and the worst-case distribution, which provides insights into the optimization dynamics.

Using implicit differentiation, we show that the stochastic gradients for MGCE are determined by worst-case distributions, and can be efficiently computed by means of a bisection method.
Experiments on benchmark datasets with deep neural networks demonstrate that MGCE improves accuracy and convergence over existing methods while producing better calibrated models, especially under label noise. Overall, the presented results show that the methodology proposed can enable to generalize cross-entropy losses in a more effective manner than existing approaches.

%\vspace{-0.1cm}
\subsubsection*{Acknowledgements}
\vspace{-0.1cm}
Funding in direct support of this work has been provided by
projects PID2022-137063NBI00, PLEC2024-011247, and CEX2021-001142-S funded
by MCIN/AEI/10.13039/501100011033 and the European Union
“NextGenerationEU”/PRTR, and \mbox{programs} BERC-2022-2025 and ELKARTEK funded by the
Basque Government. Anqi Liu is also supported by a grant from Open Philanthropy. Kartheek
Bondugula also holds a predoctoral grant (EJ-GV 2022) from the Basque Government and developed part of this research during a research stay at Johns Hopkins University.

% \balance
%\vspace{-0.2cm}
\bibliography{ref}
 \section*{Checklist}
 \begin{enumerate}

   \item For all models and algorithms presented, check if you include:
   \begin{enumerate}
     \item A clear description of the mathematical setting, assumptions, algorithm, and/or model. [Yes]
     \item An analysis of the properties and complexity (time, space, sample size) of any algorithm. [Yes]
     \item (Optional) Anonymized source code, with specification of all dependencies, including external libraries. [Yes]
   \end{enumerate}

   \item For any theoretical claim, check if you include:
   \begin{enumerate}
     \item Statements of the full set of assumptions of all theoretical results. [Yes]
     \item Complete proofs of all theoretical results. [Yes]
     \item Clear explanations of any assumptions. [Yes]     
   \end{enumerate}

   \item For all figures and tables that present empirical results, check if you include:
   \begin{enumerate}
     \item The code, data, and instructions needed to reproduce the main experimental results (either in the supplemental material or as a URL). [Yes]
     \item All the training details (e.g., data splits, hyperparameters, how they were chosen). [Yes]
     \item A clear definition of the specific measure or statistics and error bars (e.g., with respect to the random seed after running experiments multiple times). [Yes]
     \item A description of the computing infrastructure used. (e.g., type of GPUs, internal cluster, or cloud provider). [No]
   \end{enumerate}

   \item If you are using existing assets (e.g., code, data, models) or curating/releasing new assets, check if you include:
   \begin{enumerate}
     \item Citations of the creator If your work uses existing assets. [Yes]
     \item The license information of the assets, if applicable. [Not Applicable]
     \item New assets either in the supplemental material or as a URL, if applicable. [Not Applicable]
     \item Information about consent from data providers/curators. [Not Applicable]
     \item Discussion of sensible content if applicable, e.g., personally identifiable information or offensive content. [Not Applicable]
   \end{enumerate}

   \item If you used crowdsourcing or conducted research with human subjects, check if you include:
   \begin{enumerate}
     \item The full text of instructions given to participants and screenshots. [Not Applicable]
     \item Descriptions of potential participant risks, with links to Institutional Review Board (IRB) approvals if applicable. [Not Applicable]
     \item The estimated hourly wage paid to participants and the total amount spent on participant compensation. [Not Applicable]
   \end{enumerate}

 \end{enumerate}

\appendix

\input{supplement}
\end{document}

%% file: supplement.tex
\acrodef{RRM}[RRM]{robust risk minimization}
\acrodef{ERM}[ERM]{empirical risk minimization}
\acrodef{MRC}[MRC]{minimax risk classifier}
\acrodef{LR}[LR]{logistic regression}
\acrodef{SOA}[SOA]{state-of-the-art}
\acrodef{RFF}[RFF]{random Fourier features}
\acrodef{DNN}[DNN]{deep neural network}
\acrodef{SGD}[SGD]{stochastic gradient descent}
\acrodef{CE}[CE]{cross entropy}
\acrodef{MAE}[MAE]{mean absolute error}
\acrodef{GCE}[GCE]{generalized cross entropy}
\acrodef{MLP}[MLP]{multi-layer perceptron}
\acrodef{ECE}[ECE]{expected calibration error}
% If your paper is accepted, change the options for the package
% aistats2026 as follows:
%
%\usepackage[accepted]{aistats2026}
%
% This option will print headings for the title of your paper and
% headings for the authors names, plus a copyright note at the end of
% the first column of the first page.

% If you set papersize explicitly, activate the following three lines:
%\special{papersize = 8.5in, 11in}
%\setlength{\pdfpageheight}{11in}
%\setlength{\pdfpagewidth}{8.5in}

% If you use natbib package, activate the following three lines:
%\usepackage[round]{natbib}
%\renewcommand{\bibname}{References}
%\renewcommand{\bibsection}{\subsubsection*{\bibname}}

% If you use BibTeX in apalike style, activate the following line:
%\bibliographystyle{apalike}

% \begin{document}

% \runningtitle{I use this title instead because the last one was very long}


% Supplementary material: To improve readability, you must use a single-column format for the supplementary material.
\onecolumn
% \aistatstitle{Minimax Generalized Cross-Entropy\\(Supplementary Material)}

% \setcounter{equation}{20}
% \setcounter{table}{1}
% \setcounter{figure}{4}

\section{DERIVATIONS FOR THE CONVEX FORMULATION OF MGCE IN (11)}
\label{sec:derive_mgce}

The minimax risk optimization problem (7) corresponding with the $\alpha$-loss
is given as follows
\begin{equation}
\begin{array}{ccc}
      \up{V}_\beta = \underset{\up{h}}{\min} \ \underset{\up{p} \in \set{U}}{\max} \ \mathbb{E}_{\up{p}}\ell_{\beta}(\up{h}, (x,y)) =& \underset{\V{h}}{\min} \ \underset{\V{p}}{\max} &\ \beta\big(1 - {({{\V{h}}}^{1/\beta})}^\top \V{p}\big) - I_+(\V{p}) \\
     &\text{s.t.} & \sum_{y \in \set{Y}}\up{p}(x,y) = \up{p}^*(x) \ \forall \ x \in \set{X}, \\
     && \B{\tau} - \B{\lambda} \preceq \B{\Phi}^\top \V{p} \preceq \B{\tau} + \B{\lambda},
\end{array}\nonumber
\end{equation}
where each row of the vectors $\V{h} \in \mathbb{R}^{|\set{X}||\set{Y}|}$ and $\V{p} \in \mathbb{R}^{|\set{X}||\set{Y}|}$ represents a classification probability $\up{h}(x)_y$ and a probability $\up{p}(x,y)$ for distribution $\up{p} \in \set{U}$, corresponding with a sample $x \in \set{X}$ and label $y \in \set{Y}$. Each row of the matrix $\B{\Phi} \in \mathbb{R}^{|\set{X}||\set{Y}| \times m}$ represents the feature mapping vector $\Phi(x,y) \in \mathbb{R}^{m}$. Moreover, 
\begin{equation}
    I_{+}(\mathbf{p}) = 
\begin{cases}
0 & \text{if } \mathbf{p} \succeq 0 \\
\infty & \text{otherwise}.
\end{cases}\nonumber
\end{equation}

Taking the Lagrange dual of the maximization over $\V{p}$, we obtain
\begin{equation}
    \begin{array}{ccc}
         \underset{\V{h}, \B{\mu}_1,\B{\mu}_2, \B{\nu}}{\min} & \beta - (\B{\tau} - \B{\lambda})^\top\B{\mu}_1 + (\B{\tau} + \B{\lambda})^\top\B{\mu}_2 - \mathbb{E}_{\up{p}^*}\nu_x + f^*\big(\B{\Phi}^\top(\B{\mu}_1 - \B{\mu}_2) + \B{\bar\nu}\big) & \\
         s.t. & \B{\mu}_1, \B{\mu}_2 \succeq 0,
    \end{array}\nonumber
\end{equation}
where each row of the vector $\B{\nu} \in \mathbb{R}^{|\set{X}|}$ corresponds to an $x \in \set{X}$ denoted by $\nu_x$, while each row of the vector $\B{\bar\nu}\in\mathbb{R}^{|\set{X}||\set{Y}|}$ corresponds to a pair $(x, y)$ such that the value for an $x$ and any $y \in \set{Y}$ is equal to $\nu_x$. The function $f^*$ is the conjugate of $f(\V{p}) = \beta {(\V{h}^{1/\beta})}^\top \V{p} + I_{+}(\V{p})$ given by
\begin{equation}
    f^*(\V{w}) = \underset{\V{p}\succeq0}{\sup} \big(\V{w} - \beta {\V{h}^{1/\beta}}\big)^\top \V{p} = \begin{cases} 
      0 & \text{if} \ \V{w} \preceq \beta {\V{h}^{1/\beta}} \\
      \infty & \text{otherwise}.
   \end{cases}\nonumber
\end{equation}
Then, taking $\B{\mu} =\B{\mu}_1 - \B{\mu}_2$ and simplifying, we have
\begin{equation}
\label{eq:semi_final}
    \begin{array}{ccc}
       \underset{\V{h}, \B{\mu}, \B{\nu}}{\min} &  - \B{\tau}^\top\B{\mu} + \B{\lambda}^\top|\B{\mu}| - \mathbb{E}_{\up{p}^*}\nu_{x} & \\
       \text{s.t.}  & \Phi(x,y)^\top\B{\mu} + \nu_x + \beta \leq \beta \up{h}(x)_y^{1/\beta} & \forall y \in \set{Y}, x \in \set{X}.
    \end{array}\nonumber
\end{equation}
Since $\up{h}(x)_y$ represents a conditional probability of $y \in \set{Y}$ for a given $x \in \set{X}$, we have $\sum_y \up{h}(x)_y = 1$. By incorporating this condition into the constraints of the previous minimization problem, we obtain
\begin{equation}
\label{eq:final1}
    \begin{array}{ccc}
       \underset{\B{\mu}, \B{\nu}}{\min} &  - \B{\tau}^\top\B{\mu} + \B{\lambda}^\top|\B{\mu}| - \mathbb{E}_{\up{p}^*}\nu_{x} & \\
       \text{s.t.}  & \sum_{y \in \set{Y}} \Big(\frac{\Phi(x,y)^\top\B{\mu} + \nu_x}{\beta} + 1\Big)_+^\beta \leq 1, \ \ x \in \set{X}.
    \end{array}\nonumber
\end{equation}
Because the variable $\nu_x$ is implicitly defined by $\B{\mu}$, the optimization becomes equivalent to 
\begin{equation}
\label{eq:final2}
    \begin{array}{ccc}
       \underset{\B{\mu}}{\min} &  - \B{\tau}^\top\B{\mu} + \B{\lambda}^\top|\B{\mu}| - \mathbb{E}_{\up{p}^*}\{\varphi_\beta(x, \B{\mu})\}, %& \\
    \end{array}\nonumber
\end{equation}
where 
\begin{equation}
    \begin{array}{cl}
         \varphi_\beta(x, \B{\mu}) =& \underset{\nu}{\max} \ \nu  \\
         & \text{s.t.} \  \underset{y \in \set{Y}}{\sum}\Big(\frac{f(x, \B{\mu})_y + \nu}{\beta} + 1\Big)^{\beta}_{+} \leq 1.
    \end{array}\nonumber
\end{equation}
This corresponds to the optimization problem defined by (11) in the main paper.

\section{CONVERGENCE OF THE LINK FUNCTION IN (14) TO SOFTMAX FUNCTION FOR \texorpdfstring{$\beta \to \infty$}{beta to infinity}}
\label{sec:softmax_convg}
We begin by applying the limit $\underset{{\beta \to \infty}}{\lim} (1+\frac{a}{\beta})^\beta = e^a$ to the link function in (14), which gives
\begin{equation}
    \label{eq:hy_x_infty}
    \lim_{\beta \to \infty} \up{h}_\beta(x)_y = \lim_{\beta \to \infty} \Big(\frac{f(x, \B{\mu}^*)_y + \varphi_\beta(x, \B{\mu}^*)}{\beta} + 1\Big)^\beta_{+} = \mathrm{e}^{f(x,\B{\mu}^*)_y + \varphi_\infty(x, \B{\mu}^*)}.
\end{equation}
Using the result in \eqref{eq:hy_x_infty}, taking the the limit $\beta \to \infty$ of the implicit function defined in (13), we obtain
\begin{equation}
    \sum_{y \in \set{Y}}\mathrm{e}^{f(x,\B{\mu}^*)_y + \varphi_\infty(x, \B{\mu}^*)} = 1.\nonumber
\end{equation}
Solving for $\varphi_\infty$ gives the log-sum-exp normalization
\begin{equation}
    \label{eq:varphi_infty}
    \varphi_\infty(x, \B{\mu}^*) = -\log\{\sum_{y \in \set{Y}}\mathrm{e}^{f(x,\B{\mu}^*)_y}\}.
\end{equation}
Finally, substituting \eqref{eq:varphi_infty} into~\eqref{eq:hy_x_infty}, we obtain 
\begin{equation}
    \lim_{\beta \to \infty} \up{h}_\beta(x)_y = \frac{\mathrm{e}^{f(x,\B{\mu}^*)_y}}{\sum_{y \in \set{Y}}\mathrm{e}^{f(x,\B{\mu}^*)_y}}.\nonumber
\end{equation}
This corresponds to the softmax function over the classification margins $f(x, \B{\mu}^*)$.

\section{PROOF OF THEOREM 1}
\label{sec:proof_th1}
% The first result (12) is obtained by solving the entropy corresponding with $\V{p}^u \in \set{U}$. 
% On the other hand, the second result (11) is directly obtained using the relationship of the first result, and the constraint that $\V{p}^u$ is a probability distribution. In the following, we prove the first result.
We first prove the relation (16) and then establish (15) using it. For $\beta \geq 1$, and given the worst-case distribution $\up{p}_\beta \in \arg \underset{\up{p}\in\set{U}}{\max}\ \underset{\up{h}}{\min} \ \ell_\beta (\up{h}, \up{p})$, the minimax classifier $\up{h}_\beta$ is defined as a solution to the following optimization problem
\begin{align*}
    \min_{\up{h}}\ell_\beta(\up{h}, \up{p}_\beta) = \min_{\up{h}} \sum_{x,y}\up{p}_\beta(x,y)\beta(1 - \up{h}(x)_y^{\frac{1}{\beta}}) = \beta + \beta\sum_{x}\up{p}_\beta(x)\min_{\up{h}} \sum_{y}\{ - \up{p}_\beta(y|x)\up{h}(x)_y^{\frac{1}{\beta}}\}.\nonumber
\end{align*}

Therefore, this minimization problem can be solved independently for each instance $x \in \set{X}$ as
\begin{equation}
    \label{eq:h_x}
    \begin{array}{cc}
        \underset{{\V{h}}}{\min} & -\V{p}_\beta^\top\V{h}^{\frac{1}{\beta}} \\
        \text{s.t.} & \V{h} \succeq 0, \ \V{1}^\top\V{h} = 1, \\
    \end{array}
\end{equation}
where each row of the vectors $\V{h} \in \mathbb{R}^k$ and $\V{p}_\beta \in \mathbb{R}^k$ represents $\up{h}(x)_y$ and $\up{p}_\beta(y|x)$, respectively, and $\B{1}$ denotes a vector of ones.

The optimization problem in~\eqref{eq:h_x} is convex and we can utilize the Karush-Kuhn-Tucker (KKT) optimality conditions to solve it. The KKT conditions corresponding with the primal solution $\V{h}_\beta$ and the dual solution $\B{\lambda}^*, \nu^*$ are given by
\begin{align}
    \label{eq:feasibility}
    \begin{array}{cc}
        \V{h}_\beta \succeq 0, \ \B{1}^\top\V{h}_\beta = 1, \
    \B{\lambda}^* \succeq 0 & \text{(feasibility)}, \\
        \lambda^*_i{\up{h}_\beta}_i = 0 & \text{(complementary slackness)}, \\
         -\frac{1}{\beta}{\up{p}_\beta}_i{{\up{h}_\beta}_i}^{\frac{1-\beta}{\beta}} - \lambda^*_i + \nu^* = 0 \ \forall \ i \in \set{Y} & \text{(stationary condition)}, \\
    \end{array}
\end{align}
where ${\up{p}_\beta}_i$ and ${\up{h}_\beta}_i$ represent ${\up{p}_\beta}(i|x)$ and $\up{h}_\beta(x)_i$, respectively, for $i \in \set{Y}$. 

From the stationary condition in~\eqref{eq:feasibility}, we obtain
\begin{align*}
    \lambda_i^* = -\frac{1}{\beta}{\up{p}_\beta}_i{\up{h}_\beta}_i^\frac{1-\beta}{\beta} + \nu^* \ \forall \ i.
\end{align*}
Multiplying with ${\up{h}_\beta}_i$ on both sides and using the complementary slackness in \eqref{eq:feasibility}, we obtain
\begin{align*}
    0 = \Big(-\frac{1}{\beta}{\up{p}_\beta}_i{\up{h}_\beta}_i^\frac{1-\beta}{\beta} + \nu^*\Big){\up{h}_\beta}_i\nonumber
\end{align*}
\begin{equation}
\label{eq:nu_star}
    \implies \nu^* = \frac{1}{\beta}{\up{p}_\beta}_i{\up{h}_\beta}_i^{\frac{1-\beta}{\beta}} \implies {\up{h}_\beta}_i = \Big(\frac{{\up{p}_\beta}_i}{\nu^*\beta}\Big)^\frac{\beta}{\beta-1}.
\end{equation}
Since the probabilities ${\up{h}_\beta}$ satisfy $\sum_{y}{\up{h}_\beta}_i = 1$, we have
\begin{align*}
    \sum_{i \in \set{Y}} \Big({\frac{\nu^*\beta}{{\up{p}_\beta}_i}}\Big)^\frac{\beta}{1-\beta} = 1 \implies {\nu^*}^{\frac{\beta}{1-\beta}}=\frac{1}{\beta^{\frac{\beta}{1-\beta}}\sum_{i}{{\up{p}_\beta}_i}^{\frac{\beta}{\beta-1}}}.\nonumber
\end{align*}
Finally, substituting the above expression $\nu^*$ into~\eqref{eq:nu_star}, we obtain, for each $x \in \set{X}$, 
\begin{align}
    \label{eq:minimax_worst_case}
    {\up{h}_\beta}_i = \frac{{{\up{p}_\beta}_i}^{\frac{\beta}{\beta-1}}}{\sum_{i}{{{\up{p}_\beta}_i}}^{\frac{\beta}{\beta-1}}} \implies {\up{h}_\beta}(x)_y = \frac{{{\up{p}_\beta}(y|x)}^{\frac{\beta}{\beta-1}}}{\sum_{y}{{{\up{p}_\beta}(y|x)}}^{\frac{\beta}{\beta-1}}},
\end{align}
which corresponds to~(16) in the main paper.

In addition, as a consequence of \eqref{eq:minimax_worst_case}, we obtain relation (15) as follows. 
Using the fact that \mbox{$\sum_{y \in \set{Y}}\up{p}_\beta(y|x) = 1$}, and
\begin{align}
    \up{p}_\beta(y|x) = \Big(\up{h}_\beta(x)_y \sum_{y}\up{p}_\beta(y|x)^\frac{\beta}{\beta-1}\Big)^\frac{\beta-1}{\beta}\nonumber
\end{align}
from \eqref{eq:minimax_worst_case}, we have
\begin{align}
    \label{eq:intermediate_step}
    \sum_y\up{h}_\beta(x)_y^\frac{\beta-1}{\beta} = \frac{1}{\big(\sum_{y}\up{p}_\beta(y|x)^\frac{\beta}{\beta-1}\big)^\frac{\beta-1}{\beta}}.
\end{align}
Then, combining the relations in \eqref{eq:minimax_worst_case} and \eqref{eq:intermediate_step}, we achieve
\begin{align}
    \up{p}_\beta(y|x) = \frac{\up{h}_\beta(x)_y^\frac{\beta-1}{\beta}}{\sum_{y}\up{h}_\beta(x)_y^\frac{\beta-1}{\beta}}.\nonumber
\end{align}

\section{PROOF OF THEOREM 2}
\label{sec:proof_theorem2}
Since $\up{h}(x)_y$ is a probability, we have that $0 \leq \up{h}(x)_y \leq 1$ and $\sum_{y \in \set{Y}}\up{h}(x)_y = 1$. Therefore, for $\beta \geq 1$, the first inequality in (18) is achieved as follows
\begin{align*}
    \nonumber
    \up{h}(x)_y \leq \up{h}(x)_y^{\frac{1}{\beta}} \iff 1 - \up{h}(x)_y^{\frac{1}{\beta}} \leq 1-\up{h}(x)_y \iff \frac{\ell_\beta(\up{h}, (x,y))}{\beta} \leq \ell_{\up{MAE}}(\up{h},(x,y)).
\end{align*}
In the following, we prove the second inequality in (18). Consider the function $g(\up{h}(x)_y) = \beta(1 - \up{h}(x)_y^{\frac{1}{\beta}}) - (1-\up{h}(x)_y)$ that defines the difference between the $\alpha$-loss and the MAE loss for given probability $\up{h}(x)_y$. Such function is decreasing in $\up{h}(x)_y$ since the derivative $g'(\up{h}(x)_y) = 1 - \up{h}(x)_y^{\frac{1}{\beta} - 1} \leq 0$ since $0 \leq \up{h}(x)_y \leq 1$ and $\frac{1}{\beta} - 1 \leq 0$, so that the minimum value of the function is achieved at $\up{h}(x)_y = 1$. Therefore, we have that $g(\up{h}(x)_y) \geq g(1)=0$ for all $\up{h}(x)_y \in [0,1]$ and $\beta \geq 1$ which proves the second inequality in (18).

We prove inequality (19) using the function $g(\up{h}(x)_y)$ as follows. As described above, the function is decreasing in $\up{h}(x)_y$ in the interval $[0,1]$. Then, the maximum value of the function is given by $g(0) = \beta - 1$ for $\beta \geq 1$. Therefore, we have that $g(\up{h}(x)_y)\leq \beta - 1 \implies \ell_\beta(\up{h}, (x,y)) - \ell_{\up{MAE}}(\up{h}, (x,y)) \leq \beta - 1$. 

\section{PROOF OF THEOREM 3}
\label{sec:proof_theorem3}
The function $\varphi_\beta(x,\B{\mu})$ is implicitly defined as 
\begin{equation}
\label{eq:optimality_condition_supp} 
    F(x, \B{\mu}, \varphi_\beta(x,\B{\mu})) = \underset{y \in \set{Y}}{\sum}\Big(\frac{f(x, \B{\mu})_y + \varphi_\beta(x,\B{\mu})}{\beta} + 1\Big)^{\beta}_{+} =1,
\end{equation}
and the gradient $\frac{\partial\varphi_\beta(x,\B{\mu})}{\partial \B{\mu}}$ can be obtained using implicit differentiation \citep{dontchev2009implicit} as follows.

For any instance $x \in \set{X}$ and parameters $\B{\mu}$, the function $\varphi_\beta$ satisfies \eqref{eq:optimality_condition_supp}. Therefore, differentiating both sides in \eqref{eq:optimality_condition_supp} with respect to $\B{\mu}$, we have
\begin{equation}
    \underset{y \in \set{Y}}{\sum} \Big(\frac{f(x, \B{\mu})_y + \varphi_\beta(x, \B{\mu})}{\beta} + 1\Big)_+^{\beta-1} \Big(\Phi(x,y) + \frac{\partial\varphi_\beta(x, \B{\mu})}{\partial\B{\mu}}\Big) = 0.\nonumber
\end{equation}
Thus, rearranging the terms, we have
\begin{equation}
    \label{eq:implicit_gradient}
    \frac{\partial\varphi_\beta(x, \B{\mu})}{\partial\B{\mu}} = - \frac{\underset{y \in \set{Y}}{\sum}\Big(\frac{f(x, \B{\mu})_y + \varphi_\beta(x, \B{\mu})}{\beta} + 1\Big)_+^{\beta-1} \Phi(x,y)}{\underset{j \in \set{Y}}{\sum} \Big(\frac{f(x, \B{\mu})_j + \varphi_\beta(x, \B{\mu})}{\beta} + 1\Big)_+^{\beta-1}}.
\end{equation}
Moreover, using Theorem~1, we have that the worst-case distribution is
\begin{equation}
    \label{eq:worst_case_dist}
    \up{p}_\beta(y|x) = \frac{\Big(\frac{f(x, \B{\mu})_y + \varphi_\beta(x, \B{\mu})}{\beta} + 1\Big)_+^{\beta-1}}{\underset{j \in \set{Y}}{\sum} \Big(\frac{f(x, \B{\mu})_j + \varphi_\beta(x, \B{\mu})}{\beta} + 1\Big)_+^{\beta-1}}
\end{equation}
corresponding with a classifier given by parameters $\B{\mu}$. Combining \eqref{eq:implicit_gradient} and \eqref{eq:worst_case_dist}, we obtain the gradient
\begin{equation}
    \frac{\partial\varphi_\beta(x, \B{\mu})}{\partial\B{\mu}} = - \underset{y \in \set{Y}}{\sum}\up{p}_\beta(y|x)\Phi(x,y),\nonumber
\end{equation}
which corresponds to~(22).

\section{PROOF OF COROLLARY 1}
\label{sec:proof_corollary_1}
In the following, we show that the MGCE margin loss $\ell_\beta$ in \eqref{eq:margin_loss} is classification-calibrated, that is, for any $x$ and $\beta > 1$, if $\up{p}^*$ denotes the underlying distribution of the data and
\begin{equation}
    \label{eq:bayes_minimization}
    \B{\mu}^*\in\arg\min_{\B{\mu}} \sum_y \up{p}^*(y|x) \ell_\beta(\B{\mu},(x,y)),
\end{equation}
then
\begin{equation}
    \arg\max_y f(x,\B{\mu}^*)_y=\arg\max_y \up{p}^*(y|x).
\end{equation}
The minimization in \eqref{eq:bayes_minimization} corresponding with the MGCE margin loss is given as
\begin{equation}
     \underset{\B{\mu}}{\arg\min}\sum_{y} \Big(-f(x,\B{\mu})_y-\varphi_\beta(x,\B{\mu})\Big)\up{p}^*(y|x).
\end{equation}
The stationary condition for this minimization problem can be derived using the gradients in \eqref{eq:subgrad_beta} and \eqref{eq:stoc_grad}. In particular, such condition is given by
\begin{equation}
    \label{eq:optimality_condition_rev}
    \sum_{y=1}^k (\up{p}^*(y|x)-\up{p}_\beta(y|x))\Phi(x,y) = 0,
\end{equation}
where $\up{p}_\beta$ is the worst-case distribution corresponding to $\B{\mu}^*$ as given in \eqref{eq:worst_case_sol}. The above equality implies that
\begin{equation}
\up{p}^*(y|x)=\up{p}_\beta(y|x),    
\end{equation}
since the one-hot encoded vectors $\Phi(x,1), \Phi(x,2), \ldots, \Phi(x,k)$ are linearly independent. Therefore, using the relation in \eqref{eq:minmax_probs}, the corresponding minimax classifier $\up{h}_\beta(x)_y$ satisfies 
\begin{equation}
    \up{h}_\beta(x)_y \propto{\up{p}^*(y|x)^\frac{\beta}{\beta-1}},
\end{equation}
which implies that $\arg\max_y f(x,\B{\mu}^*)_y=\arg\max_y \up{p}^*(y|x)$ since $\arg\max_y \up{h}_\beta(x)_y=\arg\max_y f(x,\B{\mu}^*)_y$, and $\frac{\beta}{\beta-1}>1$.

% \section{Proof of Theorem 4}
% \label{sec:proof_cor1}
% \begin{proof}
%     Let $\mathcal{R}_\ell(h)$ and $\hat{\mathcal{R}}_\ell(h)$ denote the classification risk corresponding with clean and noisy data for a classification rule $\up{h}$. The proof of Theorem 1 in \cite{zhang2018generalized} presents a bound between the classification risk corresponding with clean and noisy distribution as
%     \begin{equation}
%         \label{eq:clean_noisy_bound}
%         \mathcal{R}_{\beta}(\up{h}) \leq \frac{\hat{\mathcal{R}}_\beta(\up{h}) - \beta\eta(k-k^{(\beta-1)/\beta})/(k-1)}{1-\eta k/(k-1)}.
%     \end{equation}
%     Using Remark~\ref{rm:1}, we have that $\hat{\mathcal{R}}_{\beta}(\hat{\up{h}}_{\beta}) \leq \hat{\up{R}}_{\beta}$ and $\mathcal{R}_{0-1}(\hat{\up{h}}^{\beta}) \leq \mathcal{R}_{\beta}(\hat{\up{h}}_\beta)$. Using these inequalities and plugging $\hat{\up{h}}_\beta$ in \eqref{eq:clean_noisy_bound}, we complete the result.
% \end{proof}
\begin{figure}[h]  % [!t] tries to place figure at top
  \centering
  % \psfrag{1234567891234}[c][c][0.7]{0-1}
  % \psfrag{1234567891235}[c][c][0.7]{CE}
\psfrag{1234567891234}[l][l][0.9]{MGCE}
     \psfrag{1234567891235}[l][l][0.9]{GCE}
     \psfrag{y}[c][c][0.9]{Accuracy (\%)}
     \psfrag{x}[c][c][0.9]{Epochs}
    \psfrag{loss}[c][t][0.9]{Loss}
    \psfrag{0}[r][r][0.8]{0}
    \psfrag{150}[r][r][0.8]{150}
    \psfrag{250}[r][r][0.8]{250}
    \psfrag{45}[r][r][0.8]{45}
    \psfrag{15}[r][r][0.8]{15}
    \psfrag{25}[r][r][0.8]{25}
    \psfrag{35}[r][r][0.8]{35}
    \psfrag{2}[c][c][0.8]{2}
    \psfrag{4}[c][c][0.8]{4}
    \psfrag{6}[c][c][0.8]{6}
    \psfrag{8}[c][c][0.8]{8}
    \psfrag{10}[c][c][0.8]{10}
    \psfrag{a}[r][r][0.8]{0}
    \psfrag{b}[r][r][0.8]{50}
    \psfrag{c}[r][r][0.8]{100}
  % ========== Single figure spanning one column ==========
  \includegraphics[width=0.5\columnwidth]{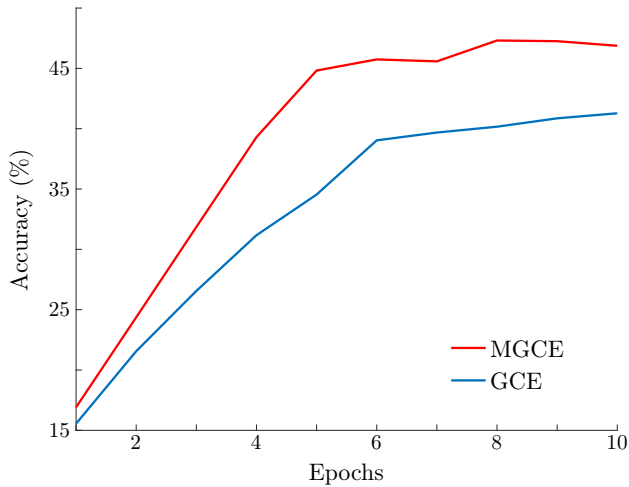}  % adjust filename
  \caption{Top-1 test accuracy on the real-world noisy dataset Clothing-1M. 
The figure shows that the proposed MGCE outperforms GCE, which underfits on this complex dataset with 1 million training samples with noisy labels.}
  \label{fig:clothing_1m}
  \vspace{0.2cm}
\end{figure}

\section{EXPERIMENTATION DETAILS AND ADDITIONAL RESULTS}
\label{sec:add_results}

\subsection{Selection of regularization parameter}
\label{sec:regularization_param_search}
Table~\ref{tb:regularization} shows the validation accuracies for multiple common values of regularization parameter $\lambda_0$ that appear in the literature (see e.g., \cite{zhang2018generalized, ma2020normalized, zhou2021asymmetric}). The table shows that the effect of the regularization parameter $\lambda_0$ is similar in the proposed method MGCE and existing techniques GCE. Based on the comparison, we use the same value of $10^{-5}$ for $\lambda_0$ in all datasets and methods.

\begin{table}[h]
\centering
\renewcommand{\arraystretch}{1.5} % Adjust row height
\resizebox{\textwidth}{!}{\begin{tabular}{|c|c|c|c|c|c|c|}
\hline
\multirow{2}{*}{Dataset} 
 & \multicolumn{2}{c|}{$\lambda_0 = 10^{-3}$} 
 & \multicolumn{2}{c|}{$\lambda_0 =10^{-4}$} 
 & \multicolumn{2}{c|}{$\lambda_0 =10^{-5}$} \\
\cline{2-7}
 & MGCE & GCE & MGCE & GCE & MGCE & GCE \\
\hline
SVHN & 80.42 $\pm$ 0.57 & 85.00 $\pm$ 1.06 & 92.41 $\pm$ 0.28 & 93.53 $\pm$ 0.13 & 94.58 $\pm$ 0.02 & 93.32 $\pm$ 0.15\\ 
CIFAR10 & 67.40 $\pm$ 1.27 & 69.32 $\pm$ 0.93 & 87.98 $\pm$ 0.53 & 88.06 $\pm$ 0.39 & 91.49 $\pm$ 0.53 & 91.11 $\pm$ 0.09 \\ 
FashionMNIST & 89.56 $\pm$ 0.53 & 89.51 $\pm$ 0.33 & 93.09 $\pm$ 0.55 & 92.74 $\pm$ 0.43 & 92.79 $\pm$ 0.50 & 92.71 $\pm$ 0.31 \\ 
\hline
\end{tabular}}
\caption{Average validation accuracies for different $\lambda_0$}
\label{tb:regularization}
\end{table}

\subsection{Experimental setup for the real-world noisy WebVision dataset}
\label{sec:real_world_setup}
The training details follow \cite{ma2020normalized}, where for each method, we train a ResNet-50 architecture \citep{he2016deep} using SGD for 250 epochs with initial learning rate 0.4, momentum 0.9, and L1 weight $10^{-5}$ and batch size 512. The learning rate is multiplied by 0.97 after every epoch of training. All the images are resized to 224 × 224. We also use the common data augmentations of random width/height shift, color jittering, and random horizontal flip are applied.

% \vspace{-2.4cm}
\subsection{Evaluation on real-world noisy Clothing-1M dataset}
 Clothing-1M \citep{xiao2015learning} is a large-scale dataset with real-world noisy labels, whose images are clawed from the online shopping websites, and labels are generated based on surrounding texts. It contains 1 million training images with noisy labels, and 15k validation images, and 10k test images with clean labels. We train a ResNet-50 architecture \citep{he2016deep} on the noisy training dataset and then validate its performance using the cleas test images. Following \cite{park2023robust}, we train the architecture using SGD for 10 epochs with initial learning rate 0.002, momentum 0.9, and L1 weight $10^{-5}$ and batch size 512. The learning rate is dropped by a factor of 10 at the halfway point of the training epochs.

\begin{figure}[t]
  \centering
  \begin{subfigure}[b]{0.48\textwidth}
    \centering
    \psfrag{12345678912345678912345}[l][l][0.9]{MGCE}
     \psfrag{12345678912345678912346}[l][l][0.9]{GCE}
     \psfrag{12345678912345678912347}[l][l][0.9]{$1-\up{V}_{1.40}$}
     \psfrag{12345678912345678912348}[l][l][0.9]{$1-\up{V}_{1.18}$}
     \psfrag{12345678912345678912349}[l][l][0.9]{$1-\up{V}_{1.05}$}
     \psfrag{12345678912345678912343}[l][l][0.9]{$1-\up{V}_{\up{MAE}}$}
     \psfrag{y}[c][c][0.9]{Test accuracy (\%)}
     \psfrag{x}[c][c][0.9]{Epochs}
    \psfrag{loss}[c][t][0.9]{Loss}
    \psfrag{50}[r][r][0.8]{50}
    \psfrag{0}[r][r][0.8]{0}
    \psfrag{60}[r][r][0.8]{}
    \psfrag{a}[r][r][0.8]{}
    \psfrag{b}[c][c][0.8]{$10^1$}
    \psfrag{c}[c][c][0.8]{$10^2$}
    \psfrag{80}[r][r][0.8]{80}
    \psfrag{100}[r][r][0.8]{100}
    \psfrag{60}[r][r][0.8]{60}
    \psfrag{40}[r][r][0.8]{40}
    \psfrag{20}[r][r][0.8]{20}
    \includegraphics[width=\linewidth]{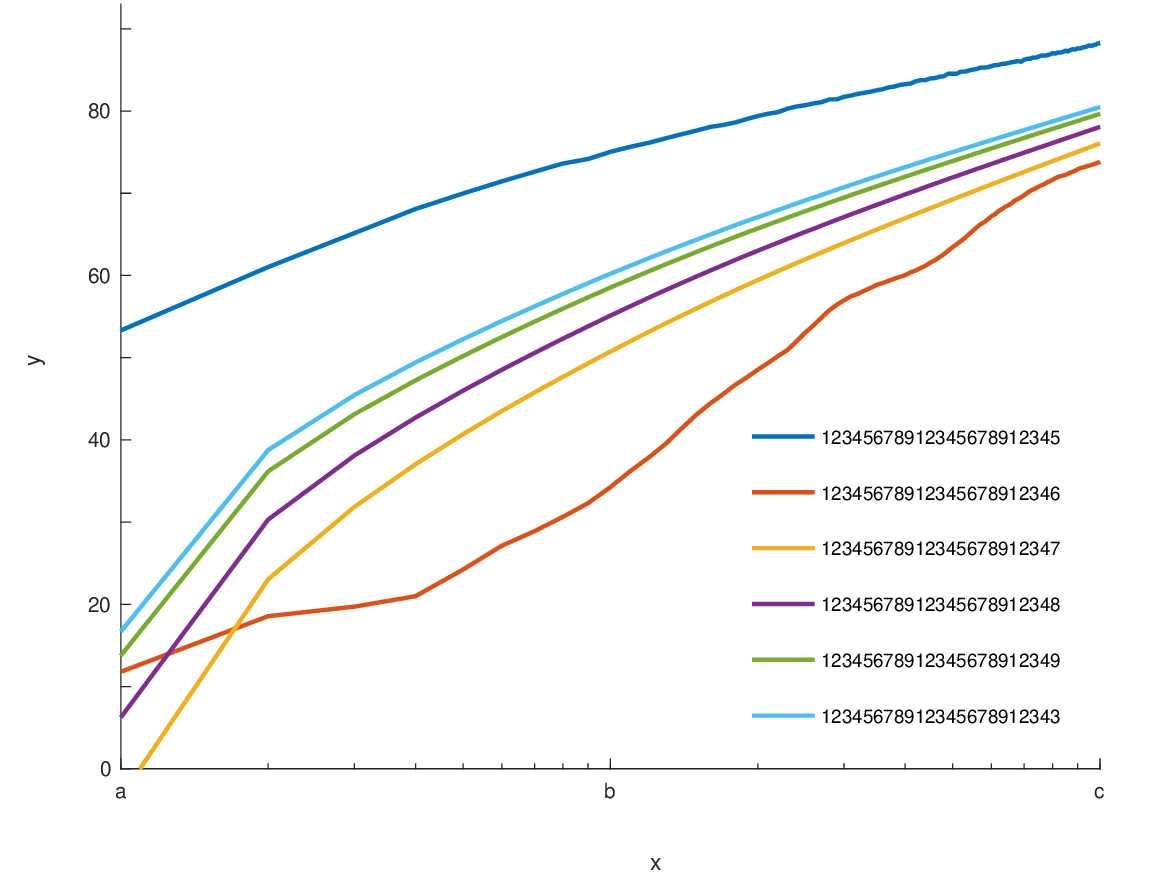} % replace with your file
    \caption{letter recognition}
    \label{fig:letter}
  \end{subfigure}
  \hfill
  \begin{subfigure}[b]{0.48\textwidth}
    \centering
    \psfrag{123456789123456781}[l][l][0.9]{MGCE}
     \psfrag{123456789123456782}[l][l][0.9]{GCE}
     \psfrag{123456789123456783}[l][l][0.9]{$\up{V}_{1.40}$}
     \psfrag{123456789123456784}[l][l][0.9]{$\up{V}_{1.18}$}
     \psfrag{123456789123456785}[l][l][0.9]{$\up{V}_{1.05}$}
     \psfrag{123456789123456786}[l][l][0.9]{$\up{V}_{\up{MAE}}$}
     \psfrag{y}[c][c][0.9]{Test accuracy (\%)}
     \psfrag{x}[c][c][0.9]{Epochs}
    \psfrag{loss}[c][t][0.9]{Loss}
    \psfrag{50}[r][r][0.8]{50}
    \psfrag{0}[r][r][0.8]{0}
    \psfrag{60}[r][r][0.8]{}
    \psfrag{a}[r][r][0.8]{}
    \psfrag{b}[c][c][0.8]{$10^1$}
    \psfrag{c}[c][c][0.8]{$10^2$}
    \psfrag{66}[r][r][0.8]{66}
    \psfrag{70}[r][r][0.8]{70}
    \psfrag{74}[r][r][0.8]{74}
    \psfrag{78}[r][r][0.8]{78}
    \psfrag{82}[r][r][0.8]{82}
    \psfrag{86}[r][r][0.8]{86}
    \includegraphics[width=\linewidth]{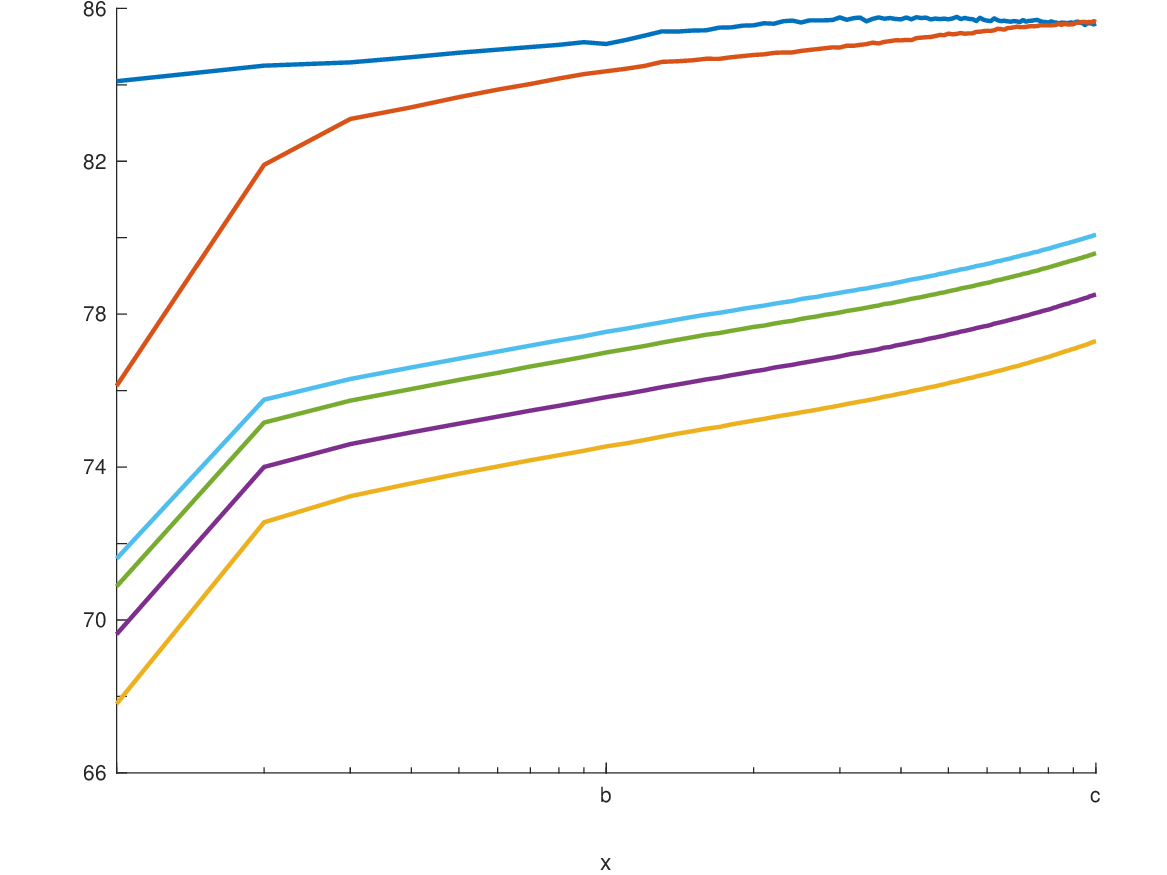} % replace with your file
    \caption{adult}
    \label{fig:adult}
  \end{subfigure}
  \vspace{0.8cm}
  \begin{subfigure}[b]{0.48\textwidth}
    \centering
    \psfrag{123456789123456781}[l][l][0.9]{MGCE}
     \psfrag{123456789123456782}[l][l][0.9]{GCE}
     \psfrag{123456789123456783}[l][l][0.9]{$1-\up{V}_{1.40}$}
     \psfrag{123456789123456784}[l][l][0.9]{$1-\up{V}_{1.18}$}
     \psfrag{123456789123456785}[l][l][0.9]{$1-\up{V}_{1.05}$}
     \psfrag{123456789123456786}[l][l][0.9]{$1-\up{V}_{\up{MAE}}$}
     \psfrag{y}[c][c][0.9]{Test accuracy (\%)}
     \psfrag{x}[c][c][0.9]{Epochs}
    \psfrag{loss}[c][t][0.9]{Loss}
    \psfrag{50}[r][r][0.8]{50}
    \psfrag{0}[r][r][0.8]{0}
    \psfrag{60}[r][r][0.8]{60}
    \psfrag{40}[r][r][0.8]{40}
    \psfrag{a}[r][r][0.8]{}
    \psfrag{b}[c][c][0.8]{$10^1$}
    \psfrag{c}[c][c][0.8]{$10^2$}
    \psfrag{66}[r][r][0.8]{66}
    \psfrag{70}[r][r][0.8]{70}
    \psfrag{74}[r][r][0.8]{74}
    \psfrag{78}[r][r][0.8]{78}
    \psfrag{82}[r][r][0.8]{82}
    \psfrag{86}[r][r][0.8]{86}
    \includegraphics[width=\linewidth]{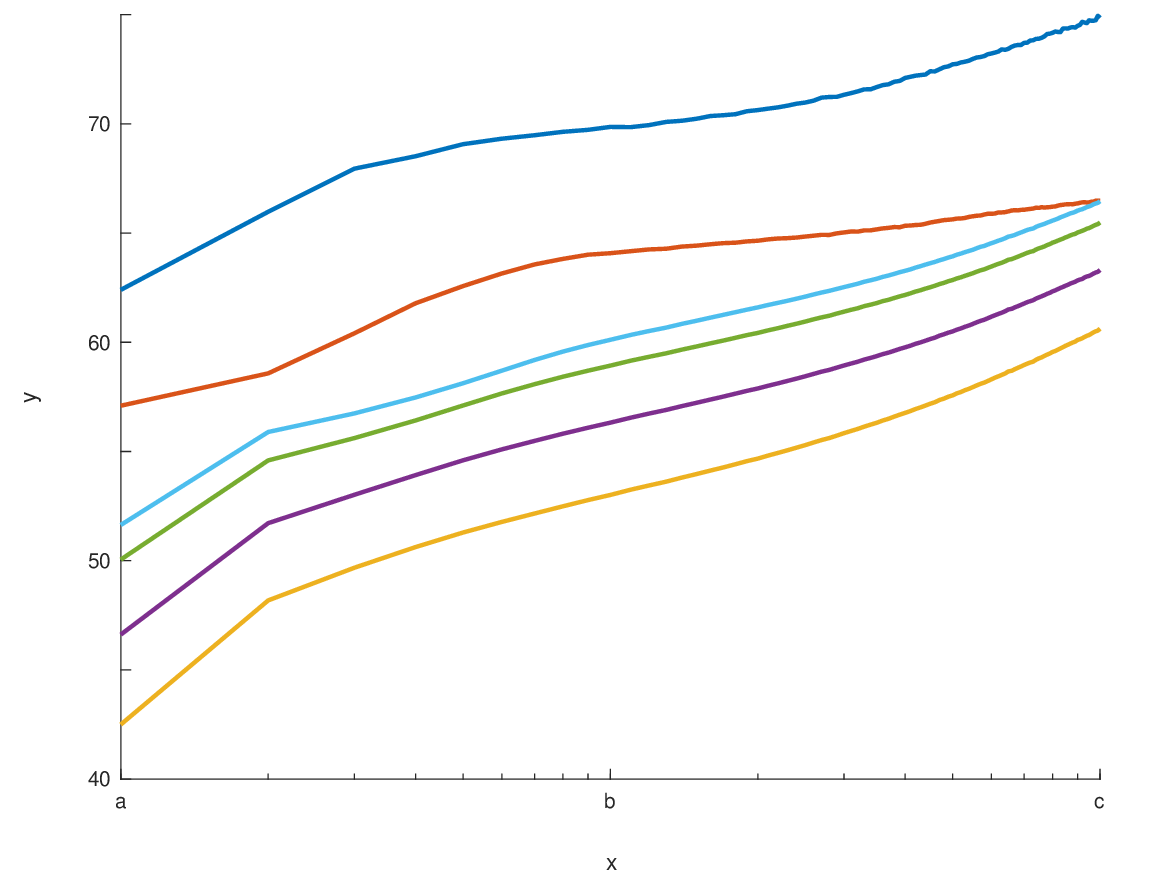} % replace with your file
    \caption{covertype}
    \label{fig:covtype}
  \end{subfigure}
  \hfill
  \begin{subfigure}[b]{0.48\textwidth}
    \centering
    \psfrag{123456789123456781}[l][l][0.9]{MGCE}
     \psfrag{123456789123456782}[l][l][0.9]{GCE}
     \psfrag{123456789123456783}[l][l][0.9]{$\up{V}_{1.40}$}
     \psfrag{123456789123456784}[l][l][0.9]{$\up{V}_{1.18}$}
     \psfrag{123456789123456785}[l][l][0.9]{$\up{V}_{1.05}$}
     \psfrag{123456789123456786}[l][l][0.9]{$\up{V}_{\up{MAE}}$}
     \psfrag{y}[c][c][0.8]{}
     \psfrag{x}[c][c][0.9]{Epochs}
    \psfrag{loss}[c][t][0.9]{Loss}
    \psfrag{40}[r][r][0.8]{40}
    \psfrag{0}[r][r][0.8]{0}
    \psfrag{50}[r][r][0.8]{}
    \psfrag{60}[r][r][0.8]{60}
    \psfrag{a}[r][r][0.8]{}
    \psfrag{b}[c][c][0.8]{$10^1$}
    \psfrag{c}[c][c][0.8]{$10^2$}
    \psfrag{70}[r][r][0.8]{}
    \psfrag{74}[r][r][0.8]{74}
    \psfrag{100}[r][r][0.8]{100}
    \psfrag{80}[r][r][0.8]{80}
    \psfrag{90}[r][r][0.8]{}
    \includegraphics[width=\linewidth]{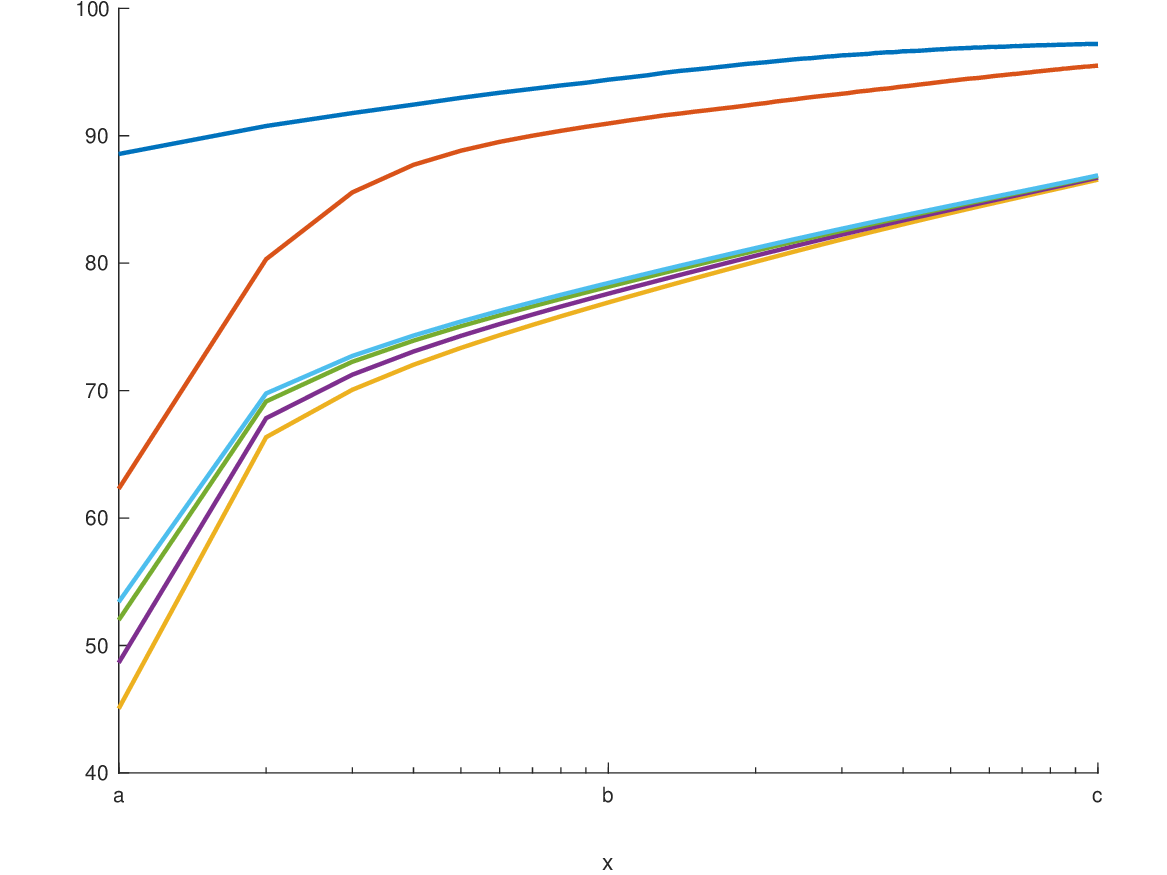} % replace with your file
    \caption{MNIST}
    \label{fig:MNIST}
  \end{subfigure}
  \caption{Average test accuracy under clean data training data obtained for multiple tabular datasets. The value of loss parameter $\beta$ is set to 1.4 for both MGCE and GCE. In addition, the figure reports the objective values $\up{V}_{1.40}$, $\up{V}_{1.18}$, $\up{V}_{1.05}$, and $\up{V}_{\up{MAE}}$ corresponding with $\beta=1.40$, $\beta=1.18$, $\beta=1.05$, and $\beta=1.00$, respectively, that provide a lower bound on the test accuracy. The figure shows the fast convergence of the proposed MGCE in comparison to GCE, and that the bounds can provide a minimum estimate on the classification accuracy.}
  \label{fig:convergence_tab}
  \vspace{-0.5cm}
\end{figure}

 Figure~\ref{fig:clothing_1m} reports the accuracy across the epochs for the test images having clean labels. The results show that the proposed MGCE method outperforms GCE method using the same value $\beta=1.4$ as in the main paper.

\vspace{-0.2cm}
\subsection{Additional experimental results using tabular datasets}
A simple \ac{MLP} classifier with 1024 hidden units was used for training. The model was optimized using SGD with a learning rate of 0.001, following the same training strategy as in the main experiments (see Section 5.1). We evaluate the model on four tabular datasets: MNIST (70K samples, 10 classes), Letter Recognition (20K samples, 26 classes), Covertype (110K samples, 7 classes), and Adult (48K samples, 2 classes).

% \begin{table*}[h!]
% \centering
% \caption{Datasets.}
% \label{tab:example_table}
% \begin{tabular}{lcc}
% \hline
% Dataset & Samples & Classes \\ 
% \hline
% MNIST & 70000 & 10 \\
% letter recognition & 20000 & 26 \\
% covertype & 110393 & 7 \\
% adult & 48842 & 2 \\
% \hline
% \end{tabular}
% % \vspace{-0.4cm}
% \end{table*}

We compare the MGCE and GCE methods in terms of test accuracy achieved along the epochs, using multiple tabular datasets and a simple MLP classifier. In addition, we assess the objective value $\up{V}_\beta$ of the proposed MGCE method in as an upper bound on the MAE classification risk of the classifier as discussed in Section 3.2 of the main paper. 

Figure~\ref{fig:convergence_tab} reports the test accuracy as well as the lower bound it (or upper bound on the classification error) across the epochs corresponding with different values of $\beta$. The results are reported for $\beta = 1.4$, that is, the same value used for all the figures in the main paper. The results show that the proposed MGCE method achieves faster convergence than the GCE method, and results in improved accuracy. Moreover, the upper bound on the classification error $\up{V}_\beta$ obtained by the MGCE method provides an estimate on the minimum test accuracy achieved by the classifier that is tighter for smaller values of $\beta$ as shown in Theorem~2.

We also evaluate the performance of MGCE and GCE across general values of the parameter $\beta \in (1.05, 11)$. We use similar setup and evaluation metrics as in Table~1 of main paper that assesses complex DNNs. In particular, we compare with the baselines provided by CE ($\beta=\infty$) and the convex 0–1 loss introduced in \cite{MazRomGrun:23}. The results in Table~\ref{tab:loss_comparison} show that the proposed MGCE method attains strong accuracy improving on the existing methods especially in multi-class datasets.
\begin{table*}[h]
\centering
\setlength{\tabcolsep}{3pt} % Adjust horizontal spacing
\renewcommand{\arraystretch}{1.2} % Adjust row height
\begin{tabular}{c|c|c|c|c|c|c|c}
\hline
\multirow{2}{*}{Dataset} & \multirow{2}{*}{Loss} 
& \multicolumn3{c|}{Test Accuracy (\%)} 
& \multicolumn3{c}{SCE (\%)} \\
\cline{3-8} 
& 
& $\eta=0.0$ & $\eta=0.2$ & $\eta=0.4$ & $\eta=0.0$ & $\eta=0.2$ & $\eta=0.4$ \\
\hline
\multirow{4}{*}{MNIST} & MGCE
     & \textbf{97.28 $\pm$ 0.07}
     & \textbf{96.15 $\pm$ 0.05}
     & \textbf{94.99 $\pm$ 0.05}
    & 0.79 $\pm$ 0.06
    & \textbf{6.52 $\pm$ 0.14}
    & 10.53 $\pm$ 0.15
\\
 & GCE
     & 97.14 $\pm$ 0.04
     & 95.74 $\pm$ 0.07
     & 94.19 $\pm$ 0.23
    & \textbf{0.37 $\pm$ 0.03}
    & 17.3 $\pm$ 0.3
    & \textbf{1.43 $\pm$ 0.07}
\\
 & CE
     & 97.17 $\pm$ 0.04
     & 95.59 $\pm$ 0.1
     & 94.02 $\pm$ 0.09
    & 0.4 $\pm$ 0.07
    & 21.18 $\pm$ 0.76
    & 39.46 $\pm$ 0.87
\\
 & MAE
     & 97.24 $\pm$ 0.03
     & 96.1 $\pm$ 0.06
     & 94.95 $\pm$ 0.13
    & 80.03 $\pm$ 0.12
    & 81.18 $\pm$ 0.22
    & 80.55 $\pm$ 0.23
\\
\hline
\multirow{4}{*}{covertype} & MGCE
     & \textbf{77.24 $\pm$ 0.09}
     & \textbf{75.83 $\pm$ 0.21}
     & \textbf{74.08 $\pm$ 0.07}
    & 1.74 $\pm$ 0.17
    & \textbf{12.58 $\pm$ 0.48}
    & \textbf{25.55 $\pm$ 0.2}
\\
 & GCE
     & 76.93 $\pm$ 0.07
     & 75.54 $\pm$ 0.14
     & 73.99 $\pm$ 0.08
    & \textbf{1.43 $\pm$ 0.14}
    & 12.89 $\pm$ 0.33
    & 25.7 $\pm$ 0.74
\\
 & CE
     & 77.28 $\pm$ 0.15
     & 75.8 $\pm$ 0.08
     & 74.11 $\pm$ 0.17
    & 3.16 $\pm$ 0.2
    & 16.41 $\pm$ 0.36
    & 29.07 $\pm$ 0.58
\\
 & MAE
     & 75.58 $\pm$ 0.09
     & 73.13 $\pm$ 0.08
     & 71.43 $\pm$ 0.22
    & 54.96 $\pm$ 0.32
    & 52.94 $\pm$ 0.1
    & 51.38 $\pm$ 0.24
\\
\hline
\multirow{4}{*}{letter} & MGCE
     & \textbf{90.63 $\pm$ 0.19}
     & \textbf{87.29 $\pm$ 0.19}
     & \textbf{83.94 $\pm$ 0.22}
    & \textbf{3.28 $\pm$ 0.3}
    & \textbf{9.97 $\pm$ 0.28}
    & 40.03 $\pm$ 0.33
\\
 & GCE
     & 86.64 $\pm$ 0.13
     & 85.19 $\pm$ 0.29
     & 82.67 $\pm$ 0.22
    & 8.44 $\pm$ 0.15
    & 25.03 $\pm$ 0.19
    & \textbf{38.9 $\pm$ 0.26}
\\
 & CE
     & 87.59 $\pm$ 0.09
     & 86.19 $\pm$ 0.21
     & 83.66 $\pm$ 0.19
    & 9.25 $\pm$ 0.14
    & 30.72 $\pm$ 0.28
    & 44.42 $\pm$ 0.18
\\
 & MAE
     & 90.58 $\pm$ 0.16
     & 87.15 $\pm$ 0.29
     & 83.09 $\pm$ 0.52
    & 83.99 $\pm$ 0.15
    & 81.5 $\pm$ 0.28
    & 77.76 $\pm$ 0.49
\\
\hline
\multirow{4}{*}{adult} & MGCE
     & \textbf{85.68 $\pm$ 0.13}
     & 84.56 $\pm$ 0.19
     & 82.63 $\pm$ 0.39
    & 7.85 $\pm$ 0.37
    & 12.13 $\pm$ 0.57
    & 6.02 $\pm$ 2.47
\\
 & GCE
     & \textbf{85.68 $\pm$ 0.09}
     & \textbf{85.33 $\pm$ 0.16}
     & \textbf{83.51 $\pm$ 0.6}
    & 1.2 $\pm$ 0.24
    & \textbf{5.69 $\pm$ 0.1}
    & \textbf{5.21 $\pm$ 1.7}
\\
 & CE
     & \textbf{85.68 $\pm$ 0.08}
     & 84.56 $\pm$ 0.21
     & 81.38 $\pm$ 0.69
    & \textbf{0.76 $\pm$ 0.1}
    & 14.3 $\pm$ 0.78
    & 23.48 $\pm$ 1.73
\\
 & MAE
     & 85.49 $\pm$ 0.17
     & 83.85 $\pm$ 0.24
     & 82.01 $\pm$ 0.38
    & 16.72 $\pm$ 1.37
    & 19.36 $\pm$ 0.21
    & 18.69 $\pm$ 1.41
\\
    \hline
\end{tabular}
\caption{Average test accuracy and static calibration error (SCE) for clean and noisy tabular datasets. The ratio $\eta$ represents the fraction of samples corrupted by symmetric label noise. We report accuracies of the epoch where validation accuracy is maximum. Our proposed MGCE method obtains strong accuracy.}
\label{tab:loss_comparison}
\end{table*}

% \vfill
% \bibliography{ref}

% \end{document}